\documentclass[11pt]{jmlr} % For arxiv, comment out for COLT
\jmlrpages{} % For arxiv, comment out for COLT
\jmlrproceedings{}{}
\title[Pareto Frontier for Online Learning of Porfolios and Quantum States]{Pushing the Efficiency-Regret Pareto Frontier for Online Learning of Portfolios and Quantum States}
\usepackage{times}
\usepackage{wrapfig}
\usepackage{tikz}
\usepackage{setspace}
\usepackage{xspace}
\usepackage{bm}
 \author{%
 \Name{Julian Zimmert} \Email{zimmert@google.com}\\
 \Name{Naman Agarwal} \Email{namanagarwal@google.com}\\
 \Name{Satyen Kale} \Email{satyenkale@google.com}\\
 \addr Google Research
 }

\newcommand{\Reg}{\operatorname{Reg}}
\newcommand{\ip}[1]{\langle#1\rangle}

\newcommand{\bbR}{\mathbb{R}}
\newcommand{\bbC}{\mathbb{C}}

\newcommand{\norm}[1]{\left\lVert#1\right\rVert}
\DeclareMathOperator*{\argmin}{\arg\,\min}

\newcommand{\ve}{\operatorname{vec}}
\newcommand{\cD}{\mathcal{D}}
\newcommand{\cG}{\mathcal{G}}
\newcommand{\cH}{\mathcal{H}}

\newcommand{\cM}{\mathcal{M}}
\newcommand{\cO}{\mathcal{O}}

\newcommand{\cX}{\mathcal{X}}
\newcommand{\cK}{\mathcal{K}}
\newcommand{\bbI}{\mathbb{I}}
\newtheorem{assumption}{Assumption}
\newcommand{\identity}{\mathbf{I}_d}
\newcommand{\tidentity}{\mathbf{I}_{\tilde d}}
\newcommand{\ev}{\operatorname{ev}}
\newcommand{\actionSet}{\mathcal{A}}
\newcommand{\cT}{\mathcal{T}}

\newcommand{\sqrtX}{X^{\frac{1}{2}}}
\newcommand{\Tr}{\text{Tr}}

\newcommand{\diag}{\operatorname{diag}}

\newcommand{\lbftrl}{\textsc{LB-FTRL}\xspace}
\newcommand{\barrons}{\textsc{Ada-BarrONS}\xspace}
\newcommand{\bisons}{\textsc{BISONS}\xspace}
\newcommand{\qbisons}{\textsc{Schr\"odinger's-BISONS}\xspace}
\newcommand{\bayes}{\textsc{Soft-Bayes}\xspace}
\newcommand{\poly}{\textrm{Poly}}
\newcommand{\Hermitian}{\cH}
\newcommand{\psds}{\cH^{d}_+}
\newcommand{\hessSet}{\cK}

\begin{document}
\maketitle

\begin{abstract}%
  We revisit the classical online portfolio selection problem. 
  % , also known as log loss problem, which has been studied for several decades.
  It is widely assumed that a trade-off between computational complexity and regret is unavoidable, with Cover's Universal Portfolios algorithm, \bayes and \barrons currently constituting its state-of-the-art Pareto frontier.
  % The exact shape of this frontier is an ongoing research question.
  In this paper, we present the first efficient algorithm, \bisons, that obtains polylogarithmic regret with memory and per-step running time requirements that are polynomial in the dimension, displacing \barrons from the Pareto frontier. Additionally, we resolve a COLT 2020 open problem by showing that a certain Follow-The-Regularized-Leader algorithm with log-barrier regularization suffers an exponentially larger dependence on the dimension than previously conjectured. Thus, we rule out this algorithm as a candidate for the Pareto frontier. We also extend our algorithm and analysis to a more general problem than online portfolio selection, viz. online learning of quantum states with log loss. This algorithm, called \qbisons, is the first efficient algorithm with polylogarithmic regret for this more general problem.
\end{abstract}

\begin{keywords}%
  Portfolio Management, Online Learning, Quantum Learning
\end{keywords}

\section{Introduction}
We study the classical online portfolio selection problem \citep{cover1991universal}. In this problem, there are $d$ assets (e.g. stocks) that an investor can invest money in on any given day. On each day, indexed by $t = 1, 2, \ldots, T$, the investor can choose a {\em portfolio} over the $d$ assets, which is a distribution of their wealth on the assets, after observing the {\em returns} (i.e. ratio of closing price to opening price) of the assets on the previous day.
The goal is to compete with the best \emph{constant-rebalanced portfolio} (CRP) in hindsight, which redistributes wealth on each day to maintain a fixed proportion in each asset. 
Importantly, we study the case without assumptions on the quality of the returns, i.e. any individual asset might suffer a total loss at any time. On any day, the wealth of the investor increases by a factor equal to the inner product between the portfolio chosen by the investor and the vector of returns for the $d$ assets. The goal is to develop algorithms that minimize the investor's {\em regret}, which is the difference between the {\em logarithm} of the total wealth earned by the investor after the $T$ days (starting with an initial wealth of \$1), and the logarithm of the total wealth earned by the best CRP in hindsight. Equivalently, the online portfolio selection problem can be seen as an instance of online convex optimization (OCO), where the loss is the negative logarithm of the inner product between the portfolio and the returns vector.

The online portfolio selection problem can be seen as a special case of a more general problem, viz. online learning of quantum states with log loss. In this problem, the goal is to learn to predict the outcome of a sequence of {\em two-outcome measurements} of an unknown quantum state on $\log_2(d)$ qubits. Without going into quantum computing jargon (we refer the reader to \citep{AaronsonCHKN18} and Appendix~\ref{sec:quantum-reduction} for a more detailed discussion of the setting), this online learning problem can be specified as follows. In each time step the learner constructs a quantum state, which is a $d \times d$ positive semidefinite Hermitian matrix of trace $1$, and in response, receives a {\em two-outcome measurement}, which is a $d \times d$ Hermitian matrix with eigenvalues in $[0, 1]$. The loss of the learner is the negative logarithm of the trace product between the quantum state generated by the learner and the measurement. The trace product can be interpreted as a probabilistic prediction of observing one of two outcomes in the measurement, and hence it is natural to use the log loss for measuring the quality of the prediction. The goal is to minimize regret with respect to the best quantum state in hindsight. It is easy to see that the online portfolio selection problem is exactly the special case of this problem where both the quantum state and loss matrices are restricted to be diagonal matrices. \citet{AaronsonCHKN18} developed regret minimizing algorithms for {\em Lipschitz} loss functions of the trace product  -- in particular, the natural log loss setting was not handled by their algorithms. 

% We now present our contributions. 
% Let $D$ be the effective dimension of the action set ($D=d$ for the online portfolios problem, and $D=d^2$ for the online learning of quantum states problem). 
Our first main contribution is the development of new algorithms, \bisons for the online portfolios problem and \qbisons for the quantum learning problem, with  regret bounds of $\cO(d^2 \log^2(T))$ and $\cO(d^3 \log^2(T))$ respectively, and $\tilde{\cO}(\text{poly}(d))$ \footnote{The $\tilde{\cO}(\cdot)$ notation suppresses polylogarithmic dependence on $T$ and $d$.} per-iteration running time. This result is noteworthy for two reasons. \bisons is the first algorithm that enjoys polylogarithmic regret with $\tilde{\cO}(\text{poly}(d))$ memory and running time per-iteration, and we show that the quantum learning problem is only slightly harder than the online portfolios problem. Technically, the \bisons algorithm operates in epochs (inspired by the \barrons algorithm of \citet{luo2018efficient}), with each epoch running a Follow-The-Regularized Leader (FTRL) algorithm with quadratic surrogate losses using the log-barrier regularizer, with an additional {\em linear bias} term added to the surrogate loss. The linear bias term is crucial to the analysis and ensures that the regret within any epoch is non-positive, while the final epoch incurs polylogarithmic regret. 

Extending the algorithm and its analysis to the quantum learning problem presents several technical challenges. First, the non-commutativity of the matrices involved makes the construction of the linear bias term non-trivial; we use semidefinite programming duality to design the linear term. Second, since the matrices are {\em complex} and Hermitian, standard convex analysis machinery such as gradients, Hessians and the intermediate value theorem need to be custom developed for the analysis. As observed earlier, the portfolios problem is a special case of the quantum learning problem when the matrices are all diagonal, and in this case \qbisons collapses to \bisons. Hence, we only give a regret bound analysis for \qbisons using the machinery developed; the bound for \bisons follows automatically.

% \na{Add some stuff highlighting the difficulty of extension - highlight the generality of the approach}

Our second main contribution is that we provide novel insights about a certain natural FTRL algorithm for the online portfolios problem. \citet{van2020open} conjectured, in a COLT 2020 open problem, that FTRL with log-barrier regularization (denoted \lbftrl) obtains the optimal $\cO(d\log(T))$ regret bound. If this is true, this would provide the first (semi-)efficient algorithm with optimal regret. We resolve the COLT 2020 open problem by disproving this conjecture with a lower bound of $\Omega(2^d\log(T)\log\log(T))$ on the regret of the \lbftrl algorithm. This result effectively removes the \lbftrl algorithm as a candidate for an optimal trade-off between complexity and regret, since our algorithm obtains superior regret (when $T\leq \exp\exp(d)$) at a significantly better run-time and memory complexity.

\begin{figure}
\label{fig:landscape}
\centering
\begin{tikzpicture}[xscale=1]
\footnotesize
        % help lines
        %\draw[step=1,help lines,black!20] (-2.95,-0.95) grid (2.95,4.95);
        % axis
        \draw[thick,->] (-5,0) -- (8,0);
        \draw[thick,->] (-2,-0.5) -- (-2,6);
        \draw[gray, fill opacity=0.1] (-1.9,1) -- (8,1);
        \draw[gray, fill opacity=0.1] (-1.9,2) -- (8,2);
        \draw[gray, fill opacity=0.1] (-1.9,3) -- (8,3);
        \draw[gray, fill opacity=0.1] (-1.9,4) -- (8,4);
        \draw[gray, fill opacity=0.1] (-1.9,5) -- (8,5);
        \draw[gray, fill opacity=0.1] (-1,0.1) -- (-1,6);
        \draw[gray, fill opacity=0.1] (1,0.1) -- (1,6);
        \draw[gray, fill opacity=0.1] (3,0.1) -- (3,6);
        \draw[gray, fill opacity=0.1] (6,0.1) -- (6,6);

        % points
        \foreach \Point/\PointLabel/\Position in {(6,1)/Universal Portfolio/above right,
        (3,5.0)/\lbftrl (upper bound)/above right,
        (-1,4)/\bayes/above right, 
        (1,5)/ONS/above right, 
        (-1,5)/EG/above right,
        (3,2)/\barrons/above right}
        \draw[fill=gray] \Point circle (0.05) node[\Position] {\PointLabel};
        \draw (3,1) circle (0.05) node[below right] {\lbftrl (conjectured)};
        \draw[fill=black] (1,2) circle (0.05) node[above right] {\textcolor{red}{\textbf{\bisons}}};
        \draw[fill=black] (3,3) circle (0.05) node[below right] {\textcolor{red}{\textbf{\lbftrl (lower bound)}}};
        
        \foreach \Tick/\TickLabel in {
        -1/d,
        1/Poly(d),
        3/Poly(d)T ,
        6/Poly(dT)
        }
        \draw[fill=black,thick] ([yshift=1.1]\Tick,0) -- ([yshift=-1.3]\Tick,0) node[below] {$\bm{\TickLabel}$};
        
        \foreach \Tick/\TickLabel in {
        1/d\log(T),
        2/Poly(d\log(T)),
        3/min(T^\gamma;Exp(d)PolyLog(T)),
        4/\sqrt{dT}PolyLog(dT),
        5/G-dependent
        }
        \draw[fill=black,thick] ([xshift=1.1]-2,\Tick) -- ([xshift=-1.3]-2,\Tick) node[left] {$\bm{\TickLabel}$};
        \node at (-4,-0.3) {Runtime};
        \node at (-2,6.3) {Regret};
    \end{tikzpicture}
        \caption{Algorithms for the portfolio problem. Worst-case regret (y-axis) in the $\poly(d)\ll T$, $T\ll\exp\exp(d)$ regime over per-step computational complexity (x-axis). Our contributions are in red. $0<\gamma<\frac{1}{2}$ is some universal constant.}
\end{figure}
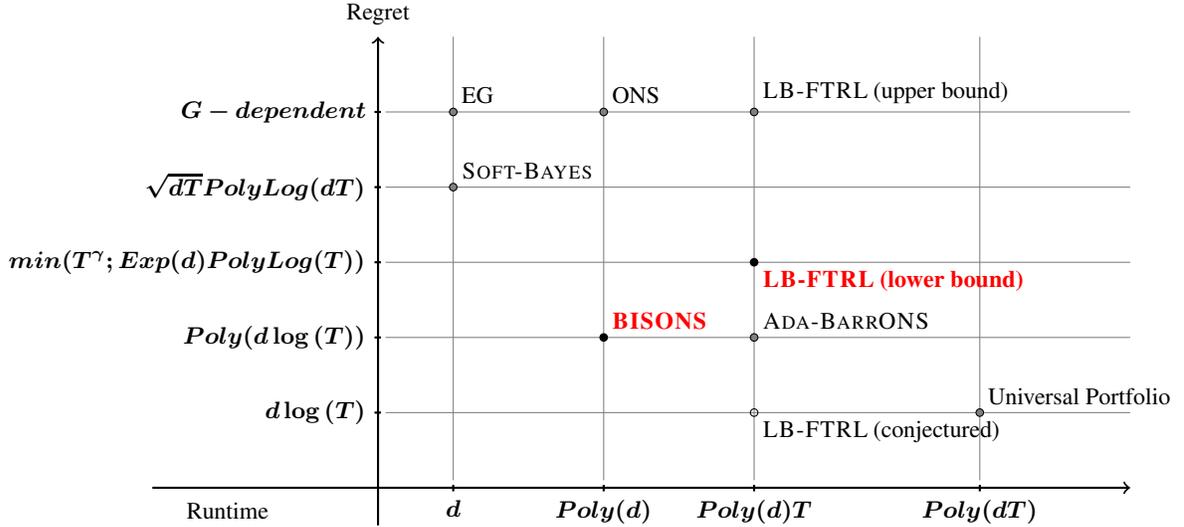

\paragraph{Related work.}
% The matrix version of the log loss problem has not been studied before, but \emph{todo: add some related papers}.
The classical online portfolios has a rich literature starting with \citet{cover1991universal}, who presented the Universal Portfolios algorithm with optimal regret. However, its fastest known implementation \citep{KalaiV00} requires $\cO(T^2(T+d)d^2)$ average per-step computation. Motivated by this inefficiency, early work \citep{agarwal2006algorithms,hazan2007logarithmic,hazan2015online} develped very efficient second order algorithms -- the primary one being Online Newton Step (ONS) -- for this problem, under the assumption that the returns of any stock are bounded away from 0 on any day. This assumption translates to a bound $G$ on the gradient of the loss function. ONS obtains $\cO(Gd\log(T))$ regret at a per-step computational complexity of $\tilde{\cO}(d^3)$. Simpler first order methods based on online gradient descent \citep{zinkevich2003online} or multiplicative weights update \citep{helmbold1998line} can also be applied to the problem, obtaining regret bounds of $\cO(G\sqrt{T\log(d)})$ and $\cO(G\sqrt{T})$ respectively, at a per step complexity of $\tilde\cO(d)$.

Since Cover's original work did not have a dependence on $G$, recent work has focused on overcoming the dependency on $G$ via both first and second order methods. 
The \bayes algorithm \citep{orseau2017soft} is a first order method that obtains $\cO(\sqrt{dT\log(d)})$ regret, while preserving linear run-time in $d$. \barrons \citep{luo2018efficient} is a second order method based on ONS and achieves $\cO(d^2\log^4(T))$ regret.
However, it requires computing the solution of log-barrier FTRL at any point, which increases its per-step complexity to $\tilde\cO(d^{2.5}T)$.

The tradeoff between regret and computational complexity described above is plotted schematically in Figure~\ref{fig:landscape}. Characterizing the Pareto frontier of this tradeoff has been a subject of study over two decades. In particular, special attention has been given to the log-barrier FTRL algorithm \citep{agarwal2005efficient}, which obtains a regret of $\cO(\min\{G^2d\log(T), d\log^d(T)\})$, but has been conjectured to obtain the optimal $\cO(d\log(T))$ regret by \citet{van2020open}.

The online learning of quantum states problem has a shorter history, being introduced by \citet{AaronsonCHKN18}. While the log loss version of the problem hasn't been studied before, it is easy to see that the log loss is 1-mixable \citep{Vovk95}, and hence Vovk's Aggregating Algorithm can be applied to the problem to obtain an algorithm with $\cO(d^2 \log(T))$ regret -- in fact, this algorithm exactly coincides with Cover's Universal Portfolios algorithm in the online portfolio setting. Implementing this algorithm however is computationally rather inefficient. 

% The online learning of density matrices problem is the special case of the online learning of quantum states problem \citep{AaronsonCHKN18} where the matrices involved are real and symmetric rather than complex and Hermitian. In the quantum setting, the learner's matrix is a {\em quantum state}, and the loss matrix is a {\em two-outcome measurement} for the quantum state. The trace product between the learner's quantum state and the measurement is a number in $[0, 1]$ that can be interpreted as a probabilistic prediction of observing one of two outcomes in the measurement. Hence, it is natural to use the log loss for measuring the quality of the prediction. \citet{AaronsonCHKN18} developed regret minimizing algorithms for Lipschitz loss functions -- in particular, the natural log loss setting was not handled by their algorithms. Our algorithm extends essentially without change to the online learning of quantum states with log loss problem, and we strongly believe that the analysis and regret bound extend to this setting as well, but currently we haven't verified this.
% \na{We can move this back -- add motivation}

\paragraph{Notation.}
For a natural number $d$ we define $[d] := \{1, 2, \ldots, n\}$, and $\Delta([d])$ to be the set of distributions over $[d]$, seen as vectors in $\mathbb{R}^d$. We denote the set of $d \times d$ Hermitian matrices by $\Hermitian^d$. We denote the set of $d \times d$ positive semi-definite Hermitian matrices by $\psds$. Through the paper $\| \cdot \|_p$ denotes the $\ell_p$ norm. Given a vector $v$ and a positive semi-definite matrix $M$, we define the semi-norm $\|v\|_{M} := \sqrt{\Tr(v^* M v)}$. Given two Hermitian matrices $X,Y$ we define the standard inner product (which is always a real number) between them as $\ip{X,Y} := \Tr(X^*Y) = \Tr(XY)$. We define additional notation required for the analysis of the quantum learning problem in the Appendix~\ref{sec:app complex definitions}. 

We use the acronyms PSD for positive semi-definite Hermitian matrices and PD for positive definite Hermitian matrices. In general, throughout the paper we denote matrices with capital letters and vectors by small letters. When denoting functions, capital letters are reserved for functions that are defined as sums of functions.  

\section{Problem setting}
\label{sec:problem-setting}
\paragraph{Online Optimal Portfolio:} The agent interacts with the environment in finite time-steps $t=1,\dots,T$.
At any time-step, the agent picks a portfolio distribution $x_t\in\actionSet=\Delta([d])$, observes a non-negative returns vector $r_t\in \bbR_+^{d}$ and suffers the log loss
\begin{align*}
    f_t(x_t) = f(x_t;r_t) := -\log(\ip{x_t,r_t})\,.
\end{align*}
Since multiplicative scaling of $r_t$ shifts the loss by a constant independent of $x_t$, the regret is unchanged if we scale $r_t$ so that it lies in $\actionSet$. The goal of the agent is to minimize its regret, defined as the cumulative loss compared to the best static action in hindsight.
\begin{align}
    \Reg = \max_{u\in\actionSet}\Reg(u)=\max_{u\in\actionSet}\sum_{t=1}^T\left(f_t(x_t)-f_t(u)\right)\,. \label{eqn: regret definition}
\end{align}
\paragraph{Quantum Learning with Log Loss:}
This problem generalizes the online optimal portfolios problem as follows. The agent's action set is $\actionSet := \{X | X \in \psds, \Tr(X) = 1\}$. The agent at every round picks a PSD Hermitian matrix $X_t \in \actionSet$, observes a PSD loss matrix $R_t$, which is assumed to be in $\actionSet$ as in the portfolios case, and suffers the log loss
\begin{align*}
    f_t(X_t) = f(X_t;R_t) := -\log(\ip{X_t,R_t})\,.
\end{align*}
The task of the agent is to mimimize regret defined analogously to \eqref{eqn: regret definition}. In the Appendix~\ref{sec:quantum-reduction}, we show that the above problem formulation captures problem of online learning of quantum states with log loss as described in \cite{AaronsonCHKN18}.

\section{Algorithm}
%
% \begin{wrapfigure}[20]{r}{8cm}
% \vspace{-1em}
\begin{algorithm2e}[h]
\setstretch{1.25}
    \textbf{input}: $T$, $B$, $\eta$, $\beta$. \\
    \textbf{initialize}:  $\forall e\in\mathbb{N}:\,p^e_0=d\bm{1}, G^e_0(\cdot) = \hat{F}^e_0(\cdot) = \eta^{-1}R(\cdot), x^e_1=u^e_1=\argmin_{x\in\actionSet}G_0^e(x)$. \\
    $e\leftarrow 1, \tau\leftarrow 1$\\
    \For{$t=1,\dots$}{
        $f_t\leftarrow$ receive from playing $x_t\leftarrow x^e_{\tau}$.\\
        $\hat f^e_\tau=\hat f_t\leftarrow$ construct according to \eqref{eq: surrogate loss}.\\
        $\hat F^e_\tau \leftarrow \hat F^e_{\tau-1} + \hat f^e_\tau$\\
        $G^e_\tau  \leftarrow G^e_{\tau-1} + g^e_\tau$, where $g^e_\tau(x) := \hat f^e_\tau(x) - \ip{x, p^e_\tau - p^e_{\tau-1}}B$\\
        
        $x^e_{\tau+1} \leftarrow  \argmin_{x\in\actionSet} G^e_{\tau}(x)$, $u^e_{\tau+1} \leftarrow \argmin_{x\in\actionSet} \hat{F}^e_{\tau}(x)$ \\
        $ \forall i\in[d]:\,p^e_{{\tau+1},i} =\max\{p^e_{\tau,i},{x^e_{\tau+1,i}}^{-1}\} $ \\
        \uIf{$\exists i:\,(2(1+6\eta)\beta)u^e_{\tau+1,i} \geq (p^e_{\tau+1,i})^{-1} $}{
              $e\leftarrow e+1,
            \tau\leftarrow 1$ \tcp{Reset the algorithm}
        }\Else{
        $\tau \leftarrow \tau+1$
        }
    }
    \caption{\bisons}\label{alg: bisons}
\end{algorithm2e}
% \end{wrapfigure}
%
In this section, we present our main algorithm \bisons (Algorithm~\ref{alg: bisons}). The algorithm is inspired by the algorithm \barrons proposed by \cite{luo2018efficient}, but improves the regret bound obtained by \cite{luo2018efficient} by a factor of $\log^2(T)$, while simultaneously and more importantly improving the run-time by factors polynomial in $T$. \bisons is the first algorithm with constant per-step computational complexity that obtains polylogarithmic regret in the portfolio problem.

The algorithm operates in \textit{epochs}, where each epoch ends when either the global time reaches $T$ or when a certain reset condition (detailed below) is met. We call an epoch \emph{completed} if it ends by reset, which sets the internal time $\tau$ of the algorithm back to $1$ and lets the algorithm forget all history. Thus, we keep only one copy of all parameters in memory and reset them to the initial values when the epoch is completed.

Let $\mathcal{T}_1 \ldots \mathcal{T}_{E} \in [1, T]$ denote the timesteps following a restart trigger event. By convention we set $\mathcal{T}_{0}=1$ and $\mathcal{T}_{E+1} := T+1$. We define an \textit{epoch} $\{\mathcal{E}_i\}$ of the algorithm as the period between successive resets of the algorithm, i.e. $\mathcal{E}_i := [\mathcal{T}_{i}, \mathcal{T}_{i+1}-1]$. Note that by definition there is no restriction over the length of these epochs and they can be of variable lengths. 

On a high level, \bisons works by approximating at every step, the true loss function $f_t(x)$ by a quadratic surrogate loss
\begin{equation}
    \begin{aligned}
    \hat f_t(x) := f_t(x_t) +\ip{x-x_t,\nabla f_t(x_t)} +\frac{\beta}{2}\ip{x-x_t,\nabla f_t(x_t)}^2
    %\\    &= (1+\beta)\ip{x,\nabla f_t(x_t)}+\frac{\beta}{2}\ip{x,\nabla f_t(x_t)}^2+c_t
    \,,\label{eq: surrogate loss}
    \end{aligned}
\end{equation}
where 
%$c_t$ is an irrelevant constant and 
$\beta\leq 1$ is an input parameter to the algorithm.
Let $e,\tau$ be the epoch and internal time of the algorithm at time $t$, then we define $x_t = x_\tau^e$ and $\hat f^e_\tau = \hat f_t$.
For reasons that become clear in section~\ref{sec: analysis}, \bisons further augments the above surrogate loss with a linear bias term, defined at every internal step $\tau$ as
%\bisons further uses a biased version of the above surrogate loss targeted at collecting negative regret, defined at every step $t$ as
\begin{equation}
\label{eq: g loss}
    g^e_\tau(x) := \hat f^e_\tau(x)-\ip{x,p^e_{\tau}-p^e_{t-\tau}}B,
\end{equation}
where $\{p^e_\tau \in \mathbb{R}^d\}$ is an auxiliary sequence maintained by the algorithm and $B$ is a bias scaling factor which is a parameter input to the algorithm. To produce the output $x^e_\tau$ \bisons runs FTRL over the biased surrogate losses, i.e.
\begin{equation}
    x^e_\tau := \argmin_{x\in\actionSet} \sum_{s=1}^{\tau-1}g^e_s(x) + \eta^{-1}R(x), \label{eq: ftrl oracle}
\end{equation}
where $\eta$ is a learning rate parameter and $R(x):=-\sum_{i=1}^d\log(x_i)$ is the log-barrier regularization. The algorithm further maintains a reference solution $u^e_\tau$ by running FTRL over the surrogate losses without bias, 
\begin{equation}
\label{eq: unbiased FTRL}
    u^e_\tau := \argmin_{x\in\actionSet} \sum_{s=1}^{\tau-1}\hat{f}^e_s(x) + \eta^{-1}R(x). 
\end{equation}
Further, the asset dependent bias $p$ is updated according to
\begin{align}
\label{eq: P update vector}
    \forall i\in[d]:\,p^e_{\tau,i} =\max\{p^e_{\tau-1,i},{x^e_{\tau,i}}^{-1}\}\,.
\end{align}
Finally, the algorithm is reset (i.e. the bias vector $p$ is reset and all previous losses are discarded) whenever
\[\exists i\in[d]:\, u^e_{\tau+1,i} > \frac{1}{2(1+6\eta)\beta)}(p^e_{\tau+1,i})^{-1}.\]

The following theorem and corollary capture our main regret bound for \bisons. We show that the total regret in any completed epoch is always non-positive and the total regret in the last uncompleted epoch is bounded. Summing the regrets over individual epochs (which is only an over-estimation of the true regret) gives the final result. 

\begin{theorem}
\label{thm: single epoch}
Assuming\footnote{Without loss of generality, we can fill up missing time-steps with $r_t=\bm{1}_d/d$, which result in constant losses.} $T \geq 110d^2$, setting the input parameters as $B=\frac{264}{5}d\log(T)$, $\eta =\frac{1}{4B}, \beta = \frac{11}{7B}$, we have that
the regret of \bisons over a completed (i.e. end triggered by the reset condition) epoch against any comparator $u:\,\min_iu_i\geq T^{-1}$ is non-positive.
Further, for the epoch that runs until the end of time $T$, the regret is bounded by $\cO(d^2\log^2(T))$.
\end{theorem}
The proof is given in Appendix~\ref{sec:bisons-analysis}, a sketch is provided at the end of Section~\ref{sec: analysis}. The following corollary is immediate:
\begin{corollary}
\label{cor: regret}
Assuming $T \geq 110d^2$, the total regret of \bisons with parameters from Theorem~\ref{thm: single epoch} is bounded by $\cO(d^2\log^2(T))$.
\end{corollary}

\paragraph{Runtime:} Note that \bisons only uses quadratic functions ($\hat f_t$ and $g_t$) and therefore a succinct representation of these functions can be maintained in time $\tilde{\cO}(d^2)$ in each iteration. Further it can be seen that the constrained minimization upto a sufficient accuracy can also be carried out in $\tilde{\cO}(\text{poly}(d))$ time (see Appendix~\ref{sec:solving-opt} for details from the more general quantum learning perspective).
% proportional to polynomial factors in $d$, $\log(T)$\footnote{It can be easily seen that accuracy needed for our analysis is $1/\poly(T)$ and the scale of the quadratic at any point is at most $T$. Therefore a runtime of at most $\poly(d, \log(T))$ is sufficient. }  

\subsection{Extension to Quantum Learning}
\label{sec: psd}
In this section, we describe the \qbisons algorithm (formally defined in the appendix as Algorithm \ref{alg: bisons (psd)}) for the quantum learning problem. \qbisons follows the same structure as \bisons, and uses the same choice of surrogate function $\hat f_t$, point played $X_t$, and comparator $U_t$ as in online optimal portfolio which are still well defined by \eqref{eq: surrogate loss}, \eqref{eq: ftrl oracle} and \eqref{eq: unbiased FTRL} respectively.

We highlight the main differences from the online optimal portfolio in this section. The main differences between the two cases firstly is that the regularizer $R$ used is the log-det-barrier, which reduces to the log-barrier for diagonal matrices: $R(X) = -\log \det(X)$. Secondly, and the primary non-trivial step in the generalization, is the appropriate definition of the biases $P_t$ and the reset condition. Analogous to Algorithm \ref{alg: bisons}, the reset condition is generalised to $U_t^e \not\preceq \frac{1}{2(1+6\eta)\beta} [P^e_t]^{-1} $, for some biases $P_t^e$ ensuring $P_t^e\succeq [X^e_s]^{-1}$ for all $s,t$ in the same epoch with $s\leq t$. This ensures that within any epoch $\hat f_t^e$ stays a valid lower bound for the comparator $U_{\tau}^e$ for that epoch. This property is summarized as Lemma \ref{lem: U bound} in the appendix. 

The main hurdle for extending our results to the quantum setting is to find a suitable bias rule $P_{\tau}^e$ that generalises \eqref{eq: P update vector}.
The goal is to construct $P_{\tau}^e$ that satisfies
$P_{\tau+1}^e \succeq P_{\tau}^e$ and $P_{\tau}^e\succeq [X_{\tau}^e]^{-1}$.
Unlike in the online optimal portfolio case, there is no canonical ``smallest'' $P_{\tau}^e$ with that property in general. Instead we choose to look for a choice satisfying these constraints that suffers a small cost of bias \footnote{See Section \ref{sec: analysis} for an explanation of what cost of bias means and how it shows up in the analysis.} 
$\sum_{t=1}^\tau \ip{X_{\tau}^e,P_{\tau}^e-P_{\tau-1}^e}$.
This objective, which can be characterized via semi-definite programming duality, leads to an optimal choice given by
\begin{align}
\label{eq: P update}
    P_{\tau+1}^e =P_{\tau}^e+ [X_{\tau+1}^e]^{-\frac{1}{2}}\left(\identity-[X_{\tau+1}^e]^{\frac{1}{2}}P_{\tau}^e [X_{\tau+1}^e]^{\frac{1}{2}}\right)_+[X_{\tau+1}^e]^{-\frac{1}{2}}\,,
\end{align}
where $(\cdot)_+$ is the operator that sets all negative eigenvalues to 0, i.e. if $M$ is a Hermitian matrix with eigendecomposition $M = U^*PU + V^*NV$, where $P$ and $N$ are diagonal matrices with the non-negative and negative eigenvalues respectively, then $M_+ = U^*PU$.
\begin{remark}
For diagonal matrices, \eqref{eq: P update} picks $P_{\tau}^e(i,i) = \max\{[X_{\tau}^e]^{-1}(i,i),P_{\tau-1}^e(i,i)\}$ and is hence a strict generalization of \eqref{eq: P update vector}. 
\end{remark}
Surprisingly, we show in the appendix that the cost of bias remains $\cO(d\log(T)B)$, so we do not pay anything for this generalization. We note that relying on the ``negative regret via linear bias'' technique used here is crucial towards obtaining this generalization. It is not clear how to use the ``negative regret by increasing learning rate'' approach used in \barrons here. We now state the theorem governing the regret for \qbisons.
\begin{theorem}
\label{thm: single epoch psd}
Assuming $T \geq 110d^2$, setting $B=\frac{264}{5}d^2\log(T)$, $\eta =\frac{1}{4B}, \beta = \frac{11d}{7B}$
the regret of \qbisons over a single epoch against any comparator $U\succeq T^{-1}\identity$ is non-positive if the end is triggered by the reset condition.
Otherwise, if the algorithm runs until the end of time $T$, then the regret is bounded by $\cO(d^3\log^2(T))$.
\end{theorem}
% \begin{proof}[Proof sketch]
% The proof structure is analogous to the proof of Theorem~\ref{thm: single epoch}.
% The main difference is that the FTRL regret over $g$ suffers $\cO(\beta^{-1}d^2\log(T))$, which is an extra $d$ factor over vanilla portfolio, because the dimension of the action set is of order $d^2$. Analogously, using the problem parameters, we get
% \[\Reg(U) = \cO(d\log(T)B) - \Omega(B^2/d)\mathbb{I}\{\text{reset triggered}\}\,,\]
% which yields the desired result with $B=\Theta(d^2\log(T))$.
% \end{proof}

Theorem \ref{thm: single epoch psd} can be used to prove the following regret bound for \qbisons yields the following corollary analogous to Corollary~\ref{cor: regret}. Missing proofs are in Appendix~\ref{sec:bisons-analysis}.
% and since the proof is the same as \ref{cor: regret}, we omit it. 
\begin{corollary}
\label{cor:qbisons-regret}
For $T \geq 110d^2$, the regret of \qbisons is bounded by $\cO(d^3\log^2(T))$.
\end{corollary}

\section{Overview of the Analysis}
\label{sec: analysis}
\paragraph{Intuition for the regret bound.}
Using quadratic surrogate losses instead of the true losses is a standard technique for improving computation complexity while preserving logarithmic regret (see the Online Newton Step (ONS) method from \cite{hazan2007logarithmic}). We use the same quadratic surrogate $\hat{f}_t$ as the ONS method (with a different choice of $\beta$). Such analyses including ONS often require that the surrogate is a lower bound for the function value of the comparator $u$, i.e. $\hat f_t(u) \leq f_t(u)$ at all time-steps. Since $u$ is unknown, this is typically enforced by ensuring lower boundedness over the entire domain. However in the case of optimal portfolio, a uniform lower bound requires $\beta$ to scale with the smallest observed gradient $G$, a quantity we wish to avoid in our bound.

\citet{luo2018efficient} observe that for any $t$, $\hat f_t(u)>f_t(u)$ only if there exists $i$ such that $u_i = \Omega(\frac{x_{t,i}}{\beta})$.
Intuitively this condition is triggered when the stock $i$ underperformed up to time $t$, thereby receiving a low weight from the algorithm, but later on recovers overproportionally.
To counter this case, our algorithm biases stocks to give them more weight according to the poorest performance they experienced.
 The bias term we introduce in our algorithm ensures a negative contribution to appear in the regret analysis.
This quantity is carefully tuned such that, if a reset happens, the regret for this phase is non-positive.
To demonstrate how the negative regret contribution appears, consider the following decomposition of the surrogate losses:
\begin{align*}
    \sum_{t=1}^\tau(\hat f_t(x_t)-\hat f_t(u))&= \sum_{t=1}^\tau(g_t(x_t)-g_t(u) +\ip{x_t-u,p_{t}-p_{t-1}}B)\\
    &=\underbrace{\Reg_g(u)}_{\text{FTRL regret bound}} +\underbrace{\sum_{t=1}^\tau\ip{x_t,p_{t}-p_{t-1}}B}_{\text{cost of bias}} \underbrace{- \ip{u,p_\tau-p_0}B}_{\text{ negative regret} }\,.
\end{align*}
The FTRL regret over the sequence of functions $g_t$ is bounded via ONS analysis. Further recall that the bias parameters $p_t$ satisfy for all $i$, $p_{ti}=\max_{s\leq t} x_{ti}^{-1}$. Therefore for all $t,i$, $p_{ti}-p_{t-1,i}\neq 0$ implies $x_{ti}=p_{ti}^{-1}$. We can now bound the \textit{cost of bias} is any epoch by
\[\sum_{t=1}^\tau\ip{x_t,p_t-p_{t-1}}B=\sum_{i=1}^d\sum_{t=1}^\tau p_{ti}^{-1}(p_{ti}-p_{t-1,i})B\leq \sum_{i=1}^d \log(p_{\tau i}/d)B\,.\]
We show in our analysis that 
$p_{ti}\leq T^2$ at all time-steps, so this
term is bounded by $\cO(d\log(T)B)$. If a reset is triggered at timestep $\tau$, then by the reset condition we have for the comparator $u_\tau$ (maintained by the algorithm), $\exists i\in[d]:\,u_{\tau i}p_{\tau i}=\Omega(\beta^{-1})$.
Hence the negative regret is of order $\Omega(\frac{B}{\beta})$, which is, given the right tuning, significantly larger than the cost of bias. We argued the above for the comparator $u_{\tau}$ maintained by the algorithm, which is the FTRL solution of the quadratic surrogate losses $\hat{f}_t$. This choice of comparator is the core reason behind our runtime improvement. We now explain why this works. 

\paragraph{Improving the run-time.}
The key to our improved runtime complexity is using the FTRL solution over the surrogate losses as comparator for the reset condition.
This computation is as costly as $x_t$, which can be done in $\cO(d^{2.5})$ arithmetic operations, in contrast to $\cO(dT)$ required by previous algorithms with optimal regret, e.g. \barrons \citep{luo2018efficient}. We first setup some auxiliary notation to simplify our argument.
Let $\ell_t(x)=\ip{x,r_t}$ be the linear reward at time $t$, then we can rewrite $f_t = h\circ \ell_t$ and $\hat f_t = \hat h_t\circ \ell_t$, where $y_t:=\ell_t(x_t)$ and $h(x),\hat h_t(x): \mathbb{R}_+ \rightarrow \mathbb{R}$ are functions defined as
\[h(x):=-\log(x)\,,\qquad\hat h_t(x) := h(y_t) + (x-y_t)h'(y_t)+\frac{\beta}{2}(x-y_t)^2h'(y_t)^2\,.\]

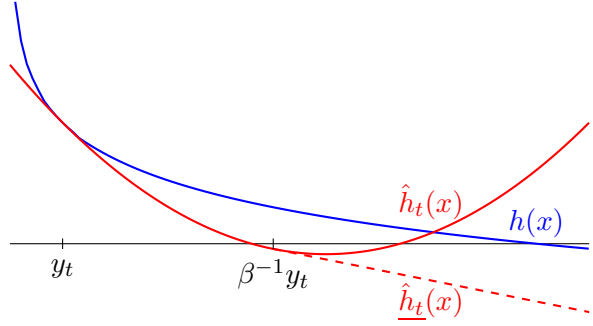
\begin{wrapfigure}{r}{0.5\textwidth}
% \begin{figure}[h!]
    \centering
\begin{tikzpicture}[xscale=7,yscale=0.7]
\draw (0,0) -- (1.1,0);
\draw[thick, blue] plot [domain=0.01:1.1,samples=100] (\x,{-ln(\x)});
\draw[thick, red] plot [domain=0.00:1.1,samples=100] (\x,{-ln(0.1)-(\x-0.1)*10+10*(\x-0.1)*(\x-0.1)});
\draw[thick,dashed,red] plot [domain=0.5:1.1,samples=2] (\x,{-ln(0.1)-(0.5-0.1)*10+10*(0.5-0.1)*(0.5-0.1)-2*
(\x-0.5)
});
\draw (0.1,-0.1) -- (0.1,0.1);
\node[below] at (0.1,-0.1)  {$y_t$};
\draw (0.5,-0.1) -- (0.5,0.1);
\node[below] at (0.5,-0.1)  {$\beta^{-1}y_t$};
\node[blue] at (1,0.4)  {$h(x)$};
\node[red] at (0.8,0.8)  {$\hat h_t(x)$};
\node[red] at (0.8,-1.1)  {$\underline{\hat h_t}(x)$};
\end{tikzpicture}
    \caption{Surrogate losses}
    \label{fig:surrogate losses}
% \end{figure}
\end{wrapfigure}
Note that both $h, \hat h_t$ are convex functions. We now define an additional function $\underline{\hat f_t}=\underline{\hat h_t}\circ \ell_t$, with $\underline{\hat h_t}(x) = \hat h_t(x)$  $x\leq \beta^{-1} y_t$, and $\hat h_t(\beta^{-1}y_t) +(x-\beta^{-1}y_t) \hat h'_t(\beta^{-1}y_t)$ otherwise. Geometrically $\underline {\hat h_t}$ coincides with $\hat h_t$ for $x$ up to $\beta^{-1}y_t$ and follows its linear extension at $x = \beta^{-1}y_t$ afterwards (see Figure~\ref{fig:surrogate losses}). From the convexity of $\hat h_t$, it follows that both $\underline{\hat h_t}$ and $\underline{\hat f_t}$ are convex.
Furthermore as shown by the following lemmma, it holds that $\underline{\hat h_t}$ is a proper lower approximation of $h$ and therefore $\underline{\hat f_t}$ is a proper lower approximation of $f$.

\begin{lemma}
\label{lem: lower surrogate}
For all $x\in (0,\infty):\,\underline{\hat h_t}(x) \leq h(x)$,
where equality holds for $x=y_t$.
\end{lemma}

The proof can be found in Appendix~\ref{sec:bisons-analysis}. We have introduced the function $\underline {\hat f_t}$ merely as a tool for the analysis. An important invariant of our algorithm that our reset condition ensures is:
\begin{lemma}
\label{lem: extended minimum}
Let $\eta\leq \min\{\frac{1}{4B},\frac{\beta}{4}, \frac{1}{63}\}$. Consider any epoch $e$ with the reset points $\mathcal{T}_{e-1} < \mathcal{T}_e\leq T$. Let $L$ represent the length of the epoch, i.e. $L = \mathcal{T}_e - \mathcal{T}_{e-1}$, we have that,
it holds that
\begin{align*}
    \min_{x \in \actionSet} \sum_{\tau=1}^L \underline{\hat f_{\tau}^e}(x)+\eta^{-1}R(x) = \sum_{\tau=1}^L \hat f_t^{e}(u_{\tau+1})+\eta^{-1}R(u_{\tau+1})\,.
\end{align*}
\end{lemma}
While we defer the proof to Appendix~\ref{sec:bisons-analysis}, the high level idea is that the reset condition ensures that $\ell_t(u_{\tau+1}) \leq \beta^{-1}y_s$ for all $s\leq \tau$. That means that the LHS is equal to the RHS around $u_{\tau+1}$. Since $u_{\tau+1}$ by definition is the minimizer of the RHS (which is a strictly conex function), hence it is a local and thereby due to convexity, a global minimizer of the LHS. We are now ready to provide a full proof sketch for Theorem \ref{thm: single epoch}.
\begin{proof}{\textbf{sketch of Theorem~\ref{thm: single epoch}}.}
Let $\tau$ denote the last time-step of any particular epoch. Then
\begin{align}
    \Reg(u)&=\sum_{t=1}^{\tau}(f_t(x_t)-f_t(u))
    \leq \sum_{t=1}^{\tau}(\underline{\hat f_t}(x_t)-\underline{\hat f_t}( u))\tag{by Lemma~\ref{lem: lower surrogate}}\\
    &\leq \max_{u'\in\actionSet} \left(\sum_{t=1}^{\tau}(\underline{\hat f_t}(x_t)-\underline{\hat f_t}(u')) -\eta^{-1}R(u')+\eta^{-1}R(u) \right) \notag\\
    &=\sum_{t=1}^{\tau}(\hat f_t(x_t)-\hat f_t( u_{\tau+1})) -\eta^{-1}R(u_{\tau+1}) +\eta^{-1}R(u)\tag{by Lemma~\ref{lem: extended minimum} }\\
    &=\Reg_g(u_{\tau+1})-\eta^{-1}R(u_{\tau+1}) +\sum_{t=1}^\tau\ip{x_t-u_{\tau+1},p_{t}-p_{t-1}}B+\eta^{-1}R(u)\,.\notag
\end{align}
We show in the detailed proof that the FTRL regret over $g$ is bounded by $\cO(\frac{d}{\beta}\log(T))$ and the regularizer is bounded by $\cO(\frac{d}{\eta}\log(T))$ due to the constraint on $u$. As discussed before, the cost of bias is bounded by $\cO(d\log(T)B)$ and the negative regret in case a reset is triggered is of order $\Omega(\frac{B}{\beta})$.
Set $\beta = \Theta(\eta)=\Theta(\frac{1}{B})$, then the regret is bounded by
\begin{align*}
    \Reg(u) = \cO(d\log(T)B) - \Omega(B^2)\mathbb{I}\{\text{reset triggered}\}\,.
\end{align*}
Finally tuning $B = \Theta(d\log(T))$ completes the proof.
\end{proof}

\paragraph{Comparison with \barrons \citep{luo2018efficient}}
\barrons uses the same surrogate loss as us, but computes $x_t$ via online mirror descent (OMD) updates with increasing learning rate. This technique is closely related to using linear biases (see \citet{FGMZ20} for a detailed discussion), however as we show via our application to the quantum learning problem (See Section \ref{sec: psd}), the latter is more flexible and additionally saves a $\log(T)$ factor in the regret.
\barrons does not use a fixed $\beta$ but instead doubles the parameter $\beta_e$ with every reset.
They ensure bounded regret by tuning the negative regret of phase $e$, such that it cancels the $\Reg_g$ term of the next phase $e+1$. Additionally, they show that the total number of epochs is bounded by $\log(T)$.
We go a step further and not only cancel the $\Reg_g$ term, but all positive regret contributions.
This allows us to use a fixed $\beta$ and saves another $\log(T)$ factor in the regret.
%At a closer look, this implies that up from some $\beta_e$ the regret of any phase with reset is negative and hence there is no need to increase $\beta_e$ further. 
%This observation allows us to run the algorithm with a constant $\beta$, which saves another $\log(T)$ factor in the regret.
Finally, our algorithm uses the FTRL solution over surrogate losses instead of the FTRL solution over the true losses for the comparator $u_t$ as run by \barrons.
This is made possible via the introduction of the auxiliary functions $\underline{f_t}$ combined with Lemma~\ref{lem: extended minimum} and yields the improvement in computational complexity.

\section{Lower bound for FTRL}
In this section, we disprove a COLT 2020 conjecture \citep{van2020open} regarding FTRL for the online portfolio selection problem.
Throughout this section, we consider FTRL with regularizer $R(x)=-\sum_{i=1}^d\log(x_i)$, simply referred to as \lbftrl. In round $t$, this algorithm plays $x_t := \arg \min_{x \in \actionSet} F_t(x)$,
where $F_t(t) := \eta^{-1} R(x) + \sum_{\tau=1}^{t-1} f_t(x)$ and $\eta > 0$ is a constant hyperparameter. This is in some sense a natural choice, since the adversary can ``force'' the player to operate with this regularization by picking $r_i=\bm{e}_i$ for $i\in[d]$.
Indeed \citet{van2020open} conjectured that FTRL obtains the optimal bound of $\cO(d\log(T))$, while we prove an exponentially worse lower bound of $\Omega(2^d\log(T)\log\log(T))$. Our main theorem, stated in a slightly abstract fashion for notational convenience, is the following (all missing proofs appear in Appendix~\ref{sec:ftrl-lb}):
\begin{theorem}
\label{thm: general lower bound}
Let $\cT > 0$ and let $\bm{t}_1,\dots,\bm{t}_{\cT}$ and $\bm{o}_1,\dots,\bm{o}_{\cT}$,  be sequences of target vectors and associated returns vectors in $\Delta([d])$, which satisfy
$\forall j<i:\,\ip{\bm{t}_i,\bm{o}_j} = \Omega(1/\poly(d))$, and $\forall i:\,\ip{\bm{t}_i,\bm{o}_i} = 0$,
then there exists $T_0 = \poly(\cT,d)$, such that for any $T>T_0$ the regret of \lbftrl against the sequence of reward vectors $r_t$ generated by Algorithm~\ref{alg: general bad sequence} (Appendix~\ref{sec:ftrl-lb}) is lower bounded by
\begin{align*}
    \Reg = \Omega(\cT\log(T)\log\log(T))\,.
\end{align*}
\end{theorem}
\paragraph{Remark.} This lower bound extends easily to the quantum version of \lbftrl which uses the log-det regularizer via the observation that when all the loss matrices $R_t$ are diagonal, log-det regularized $\lbftrl$ reduces to vanilla (log barrier regularized) \lbftrl.

\paragraph{Lower bound proof sketch.}
First, we note that the action set $\Delta([d])$ lies in a $(d-1)$-dimensional subspace of $\bbR^d$. For technical reasons, it will be convenient to work with a full dimensional action set with non-zero volume.
 % hence the lower bound part cannot be directly applied.
Hence, we define the projection operator $\Pi:\bbR^{d}\rightarrow\bbR^{d-1}$ with kernel $\bm{c}=\frac{1}{d}\bm{1}_d$ and $\Pi^{-1}$ its inverse mapping into $\actionSet$.\footnote{
Let $U$ be a $(d-1) \times d$ matrix whose columns form an orthonormal basis for the subspace orthogonal to $\bm{c}$. Then $\Pi$ can be defined as $\Pi x = U^* x$, and $\Pi^{-1}$ as $\Pi^{-1} v = Uv + \bm{c}$.} Thus $\actionSet$ gets mapped to $\Pi\actionSet$, which has non-zero volume in $\mathbb{R}^{d-1}$. In a slight overload of notation, we consider $f^\Pi(\tilde x;y)$ as a function with argument $\tilde x\in\Pi\actionSet$ by the identity 
\begin{align*}
    f^\Pi(\tilde x;r)=f(\Pi^{-1}\tilde x;r)= -\log(\ip{\Pi^{-1}\tilde x, r})=-\log(1/d+\ip{\tilde x, \Pi r})\,,
\end{align*}
and use $\nabla_\Pi f_t(x)=\nabla f^\Pi(\Pi x; r_t)$, $\nabla^2_\Pi f_t(x)=\nabla^2 f^\Pi(\Pi x; r_t)$ as shorthand notation for the gradient and Hessians with respect to the above definition of $f_\Pi$. We define $\nabla_\Pi F_t(x)$ and $\nabla^2_\Pi F_t(x)$ analogously.

% Hence the gradient and Hessian of the loss are respectively
% \begin{align*}
%    & \nabla f(x;r) = -\frac{r}{\ip{x,r}}\,,\qquad
%    &&\nabla_\Pi f(x;r) = -\frac{\Pi r}{\frac{1}{d}+\ip{\Pi x,\Pi r}}=-\frac{\Pi r}{\ip{x,r}}\\
%     &\nabla^2 f(x;r) = \nabla f(x;r)\nabla f(x;r)^\top\,, &&\nabla^2_\Pi f(x;r) = \nabla_\Pi f(x;r)\nabla_\Pi f(x;r)^\top\,.
% \end{align*}

The lower bound rests on the following key lemma, which shows that the regret of \lbftrl is lower bounded by a certain quantity which also appears in the {\em upper bound} for FTRL in the standard analysis; so this quantity controls the regret tightly. This is, to the best of our knowledge, a novel idea and is crucial in showing that \lbftrl does not obtain $\tilde{\cO}(d\log(T))$ regret in the portfolio problem.
\begin{lemma}
\label{lem: lb-ftrl regret lower bound}
The regret of \lbftrl is lower bounded as follows:
\[
    \Reg = \Omega\left(\sum_{t=1}^T\norm{\nabla_\Pi f_t(x_t)}^2_{(\nabla^2_\Pi F_t(x_t))^{-1}} \right).
\]
\end{lemma}
% The proofs of the above lemmas are standard FTRL analysis and are deferred to the appendix \footnote{For the quantum case in Lemma \ref{lem: ftrl regret upper bound} gradients and Hessians are specially defined. This is explained in the proof.}.
% Whenever $\norm{\nabla g_t(x_t)} = \cO(\poly(T))$ and $\nabla^2 G_t(x_t) \succeq \beta\sum_{s=1}^t\nabla g_t(x_t)\nabla g_t(x_t)^\top$, the norm term of the upper bound simplifies to $\cO(\beta^{-1}\tilde d \log(T))$ by standard ONS analysis.
% Using that the regret is also lower bounded by the norm term is, to the best of our knowledge, 

% Our lower bound relies on Lemma~\ref{lem: ftrl regret lower bound}, which requires the action set to have non-zero volume.

%

%\paragraph{Remark.} We did not optimize this theorem for the condition $T>d^{16}$, it likely holds for significantly smaller polynomials as well. 

We now give a high level intuition of why a lower bound of $\cT\log(T)$ is possible. The extra $\log\log(T)$ factor requires a careful layering construction that is deferred to the appendix.
The main idea of algorithm~\ref{alg: general bad sequence} is to let the agent sequentially visit each of the points $(\bm{t}_i)_{i=1}^\cT$ for $T^\alpha$ steps ($0<\alpha<1$ is some fixed parameter), and the agent receives the return $\bm{o}_i$ at point $\bm{t}_i$.
Since $\bm{t}_i$ are on the boundary of the domain, which the agent cannot reach exactly, we refer to visiting $\bm{t}_i$ if the agent plays $\bm{t}_i'=(1-T^{-\alpha})\bm{t}_i+T^{-\alpha}\bm{c}$, which is the target pulled towards the center by $T^{-\alpha}$.

Let us first assume that this is possible and that we only need to care about these returns in the Hessian. By Lemma~\ref{lem: lb-ftrl regret lower bound}, the regret is lower bounded by
\begin{align*}
    \Reg =\Omega\left( \sum_{t=1}^T\norm{\nabla_\Pi f_t(x_t)}^2_{(\nabla^2_\Pi F_t(x_t))^{-1}}\right)=\Omega\left( \sum_{t=1}^T\frac{\norm{\nabla_\Pi f_t(x_t)}^2}{\Tr(\nabla^2_\Pi F_t(x_t))}\right)
\end{align*}
During the $T^\alpha$ times we visit $\bm{t}_i$ and receive $\bm{o}_i$, the term $\norm{\nabla_\Pi f(\bm{t}_i';\bm{o}_i)}^2$ is of order $T^{2\alpha}$ (ignoring dimension dependence), since it scales with $\ip{(1-T^{-\alpha})\bm{t}_i+T^{-\alpha}\bm{c},\bm{o}_i}^{-2}=T^{2\alpha}\ip{\bm{c},\bm{o}_i}^{-2}$.
The trace in the denominator (ignoring the regularizer) after the $m$-th visit of $\bm{t}_i$, is 
\[
\sum_{j< i}T^\alpha \norm{\nabla_\Pi f(\bm{t}_i';\bm{o}_j)}^2+m\norm{\nabla_\Pi f(\bm{t}_i';\bm{o}_i)}^2 = \cO(\cT \poly(d) T^\alpha)+m\norm{\nabla_\Pi f(\bm{t}_i';\bm{o}_i)}^2\,,
\]
which uses $\ip{\bm{t}_i', \bm{o}_j}=\Omega(1/\poly(d))$ for $i<j$. 
We can assume that $T$ is large enough that $T^{\alpha/2} = \Omega(\poly(d))$, so for any $m>T^{\alpha/2}$, the denominator is of order $\cO(m\norm{\nabla_\Pi f(\bm{t}_i';\bm{o}_i)}^2)$.
Hence the stability is approximated by
\begin{align*}
    \sum_{i=1}^\cT \sum_{m=T^{\alpha/2}}^{T^\alpha} \frac{\norm{\nabla_\Pi f(\bm{t}_i';\bm{o}_i)}^2}{m\norm{\nabla_\Pi f(\bm{t}_i';\bm{o}_i)}^2}=\Omega(\cT \log(T))\,.
\end{align*}
This shows that the stability is large if the agent's trajectory can be controlled. In fact, this is possible without increasing the trace of the Hessian significantly.
To ensure that the agent visits the points $\bm{t}'_i$, we interleave the $\bm{o}_i$ returns by additional \emph{movement-returns} $r_t$, which satisfy $\norm{\nabla\Pi r_t}=\cO(T^{-\frac{1}{2}})$. 
Since the contribution to the Hessian is quadratic, the cumulative contribution to the Hessian trace of all movement steps does not exceed $\cO(\poly(d))$, which is negligible in the argument above.
Finally, one needs to show that the required number of \emph{movement-returns} is small enough such that the sequence does not exceed $T$ time steps. In our detailed proof, we show that this always holds for $\alpha=\frac{1}{8}$ and sufficiently large $T$.
\paragraph{Exponential lower bound for \lbftrl.}
Equipped with Theorem~\ref{thm: general lower bound}, we are ready to derive an exponential lower bound for \lbftrl.
 We define the following sequence of target point sets for any $k\in[d-1]$: $\bm{\cT}_k := \left\{\frac{1}{k}x\,\big|\, x\in\{0,1\}^d, \norm{x}_1=k \right\}$,  i.e. the sets where exactly $k$ components of the vector are non-zero, and these are of equal size.
 Define the combined sequence by adding the sets in increasing order of $k$, with arbitrary ordering within a set $\bm{t}_1,\dots,\bm{t}_\cT= (\bm{t}\in\bm{\cT}_1 ),\dots,(\bm{t}\in\bm{\cT}_{d-1})$.
 For each $i\in[\cT]$, define the associated returns vector by $\bm{o}_i := \frac{1}{d-\norm{\bm{t}_i}_0}(\bm{1}_d - \norm{\bm{t}_i}_0 \bm{t}_i)$, i.e. the complement vector that is non-zero iff $\bm{t}_i$ is zero, normalized so that it lies in $\Delta([d])$.

\begin{lemma}
\label{lem: good sequence}
For all $i<j$, it holds $\ip{\bm{t}_i,\bm{o}_j}=\Omega(1/d^2)$ ,
as well as $\ip{\bm{t}_i,\bm{o}_i}=0$.
\end{lemma}
\begin{proof}
The second equality follows trivially by construction. For the first observe that for any $j<i$, the number of non-zero components in $\bm{t}_j$ does not exceed the number of non-zero components in $\bm{t}_i$. That means that if $\bm{t}_i$ has $k$ non-zero entries, then $\bm{o}_j$ has at least $d-k$ non-zero entries. Since $\bm{t}_j\neq\bm{t}_i$,
$\bm{o}_j\neq\bm{o}_i$, there is at least one component of non-zero values overlapping. Finally all non-zero components are least of size $\frac{1}{d}$, which completes the proof. 
\end{proof}
\begin{corollary}
The worst-case regret of \lbftrl for any $T>\poly(2^d)$ is $\Omega(2^d\log(T)\log\log(T))$.
\end{corollary}
\begin{proof}
Combine Lemma~\ref{lem: good sequence} with Theorem~\ref{thm: general lower bound} and observe that the constructed sequence is of length $2^d-2$.
\end{proof}

\section{Conclusion}
We have presented \bisons, the first algorithm with $\tilde{\cO}(\text{poly}(d))$ memory and per-step running time that obtains near optimal regret in the optimal portfolio problem without any assumptions on the gradient. Further, we have shown that key techniques in our algorithm \bisons can be adapted to work with the more general setting of quantum learning with log loss as well, at an additional factor of $d$ in the regret. 

Further, we showed that previous conjectures about \lbftrl are wrong and that the worst-case regret of \lbftrl is at least of order $2^d\log(T)\log\log(T)$. In the natural regime of $T\ll\exp(\exp(d))$ the regret of \bisons outperforms \lbftrl at a significantly lower run-time and memory complexity.  
Therefore we practically eliminate \lbftrl as a candidate for optimal trade-off between regret and computational complexity.

% Acknowledgments---Will not appear in anonymized version
% \acks{We thank a bunch of people.}

\bibliography{mybib}

\appendix
\section{Quantum Learning: Preliminaries}

\subsection{Reductions for Online Learning of Quantum States with Log Loss}
\label{sec:quantum-reduction}
In this section we describe the online learning of quantum states problem as described in \citep{AaronsonCHKN18}, and show that when the loss function is the log loss (and more generally, the KL-divergence), the problem can be cast in the form in Section~\ref{sec:problem-setting}. Recall that a {\em quantum state} on $\log_2 d$ qubits is a $d \times d$ Hermitian PSD matrix of trace 1. A {\em two-outcome measurement} is a $d \times d$ Hermitian matrix with eigenvalues in $[0, 1]$. When a quantum state $X$ is measured using a two-outcome measurement $E$, the result is a Bernoulli random variable with probability of $1$ being $\langle X, E\rangle$.

\citet{AaronsonCHKN18} formulated the problem of online learning of quantum states as follows. In each round $t$, the learner constructs a quantum state $X_t$. In response, nature provides a two-outcome measurement $E_t$ and a value $b_t \in [0, 1]$. The value $b_t$ may be considered to be an approximation of $\langle X, E_t \rangle$ for some unknown quantum state $X$ that we're trying to learn, or it can be thought of as the outcome in $\{0, 1\}$ of measuring the state $X$ using $E_t$. However, as is standard in online learning, the pair $(E_t, b_t)$ doesn't have to be consistent with any quantum state. The quality of the learner's prediction is given by a loss function $\ell: \mathbb{R}_+ \times [0, 1] \rightarrow \mathbb{R}$, and the loss in round $t$ is computed as $\ell(\langle X_t, E_t\rangle, b_t)$. The goal is to minimize the regret, defined in the usual way as
\[\Reg = \sum_{t=1}^T \ell(\langle X_t, E_t\rangle, b_t) - \min_{\text{quantum state } X} \sum_{t=1}^T \ell(\langle X, E_t\rangle, b_t).\]
We now show that in either of the following two settings, the problem can be recast in the form given in Section~\ref{sec:problem-setting}.
\paragraph{Setting 1:} $b_t \in \{0, 1\}$, and $\ell$ is the log loss, i.e. $\ell(p, b) = -\log(b p + (1-b)(1-p))$. \\
In this case, note that by setting the loss matrix to be $R_t := b_t E_t + (1-b_t)(\identity - E_t)$, we have $-\log(\langle X, R_t\rangle) = \ell(\langle X, E_t\rangle, b_t)$ for any quantum state $X$. This completes the reduction to the form in Section~\ref{sec:problem-setting}.

\paragraph{Setting 2:} $b_t \in [0, 1]$, and $\ell$ is the KL-divergence, i.e. $\ell(p, b) = b\log(\frac{b}{p}) + (1-b)\log(\frac{1-b}{1-p})$. \\
In each round $t$, sample a Bernoulli random variable $y_t$ with probability of $1$ being $b_t$. Then, setting $R_t = \frac{y_t}{b_t}E_t + \frac{1-y_t}{1-b_t}(\identity - E_t)$, it is easy to check that $\mathop{\mathbb{E}}_{y_t}\left[-\log(\langle X, R_t\rangle)\right] = \ell(\langle X, E_t\rangle, b_t)$ for any quantum state $X$. This completes a {\em randomized} reduction to the form in Section~\ref{sec:problem-setting}. Note that setting 1 is the special case of this setting when $b_t \in \{0, 1\}$, and in this case the randomized reduction becomes deterministic and coincides with the reduction described for setting 1.

% To this end recall that the online learning of quantum states problem entails at every step the adversary selecting a measurement $E_t$, the learner providing a quantum state $\rho_t$ and the learner further observing the outcome of the measurement $Y_t$. The loss suffered by the learner in the log loss case is $-\log(Y_t \ip{E_t, \rho_t} + (1-Y_t) (\ip{\identity - E_t, \rho_t}))$. Redefining the returns vector as $R_t := Y_t E_t + (1-Y_t)(\identity - E_t)$, we get that the suffered loss is $-\log(\ip{R_t, \rho_t})$. This concludes the reduction.

\subsection{Preliminary Notation, Definitions and Useful Properties}

In this section, we collect some basic notation, definitions and useful properties which allow for the extension of the usual concepts in online convex optimization to the case when the domain is Hermitian matrices. 

\paragraph{Notation} Recall that, we denote the set of $d \times d$ Hermitian matrices by $\Hermitian^d$, the set of $d \times d$ positive semi-definite Hermitian matrices by $\psds$. Further, given two Hermitian matrices $X,Y$ we define the standard inner product between them as $\ip{X,Y} := \Tr(X^*Y) = \Tr(XY)$. Given a $d \times d$ matrix $M$, we use $M_{\mathrm{flat}}$ denotes its canonical flattening which is $d^2$ dimensional vector obtained by serializing the columns of $M$. We further define $\overrightarrow{\Hermitian}_{\mathrm{flat}}^d = \{ \overrightarrow{M} | M_{\mathrm{flat}} \in \Hermitian^d\}$. Note that $\Hermitian_{\mathrm{flat}}^d$ forms a subspace in the vector space $\bbC^{d^2}$. Let $\mathrm{dim}_{\Hermitian^d}$ be the dimension of the subspace and let $\Pi_{\Hermitian^d} \in \bbC^{d^2 \times \mathrm{dim}_{\Hermitian^d}}$ be a projection operator from $\bbC^{d^2}$ on to this subspace. For convenience due to repeated usage through the paper given a $d \times d$ Hermitian matrix $M$, we define its vectorization as the projection of its canonical flattening, i.e.
\[ \overrightarrow{M} = M_{\mathrm{flat}}^*\Pi_{\Hermitian^d} \]
% Further define the subset of matrices $\hessSet \subseteq \Hermitian^{d^2}$ whose Kernel contains the orthogonal complement of $\overrightarrow{H}^d$, i.e. 
% \[ \hessSet = \{ M | M \in \Hermitian^{d^2}, \forall  \overrightarrow{v} \in [\overrightarrow{\Hermitian}^d]^{\dagger}, Mv = 0\}\]
Further define the subset of matrices $\hessSet := \Hermitian^{\mathrm{dim}_{\Hermitian^d}^2}$. 
% whose Kernel contains the orthogonal complement of $\overrightarrow{H}^d$, i.e. 
% \[ \hessSet = \{ M | M \in \Hermitian^{d^2}, \forall  \overrightarrow{v} \in [\overrightarrow{\Hermitian}^d]^{\dagger}, Mv = 0\}\]
This set of matrices will be used to define the Hessian matrices in the complex case. 
% For a matrix $K \in \hessSet$  which is invertible in the subspace $\overrightarrow{\Hermitian}^d$, we will by an abuse of definition use $K^{-1}$ to denote its pseudo-inverse, since we will only be concerned with the action of $K^{-1}$ on the subspace $\overrightarrow{\Hermitian}^d$.

\paragraph{Definitions and Useful Properties:} Note that as such the gradient of a real-valued function over the set of complex numbers does not exist (unless the function is a constant function). However, since we are dealing with the set of Hermitian matrices, we can define appropriate notions of a gradient and Hessian that allows for the same treatment of the proof as in the case of real matrices.    

\begin{definition}
We say a function $f:\Hermitian^d\rightarrow \bbR$ admits the gradient function $\nabla f:\Hermitian^d\rightarrow\Hermitian^d$, if for all PD $X$ it satisfies
\begin{align*}
    \forall Y\in\Hermitian^d:\,\lim_{h\rightarrow 0}\frac{f(X+h(Y-X))-f(X)}{h} = \ip{(Y-X),\nabla f(X)}\,.
\end{align*}
Additionally we that say the function $f:\Hermitian^d\rightarrow \bbR$ also admits the Hessian function $\nabla^2 f:\Hermitian^d\rightarrow\hessSet$ if for PD $X$ it satisfies
\begin{align*}
    \forall Y\in\Hermitian^d:\,\lim_{h\rightarrow 0}\frac{\ip{Y-X,\nabla f(X+h(Y-X))-\nabla f(X)}}{h} = (\overrightarrow{Y}-\overrightarrow{X})^*\nabla^2 f(X) (\overrightarrow{Y}-\overrightarrow{X})
\end{align*}
\end{definition}

A simple property to note is that the admissibility of gradient and Hessian function is additive, i.e. if $f$ admits gradient function $\nabla f$ and $g$ admits $\nabla g$, then $f+g$ admits $\nabla f + \nabla g$. The same property holds for the Hessian.  

\begin{lemma}[Chain rule.]
\label{lemma: complex chain rule}
Let $X$ be a PD matrix and $Y$ be a PSD matrix. Further define the function $g(t) = t Y + (1-t) X$. If $f$ admits the gradient function $\nabla f$, we have that for any $\alpha \in [0,1)$
\begin{align*}
    (f\circ g)'(\alpha) = \ip{Y-X, \nabla f(g(\alpha))}\,.
\end{align*}
If $f$ further admits  Hessians, then 
\begin{align*}
    (f\circ g)''(\alpha) = (\overrightarrow{Y}-\overrightarrow{X})^*\nabla^2 f(g(\alpha)) (\overrightarrow{Y}-\overrightarrow{X})\,.
\end{align*}
\end{lemma}
\begin{proof}
\begin{align*}
    (f\circ g)'(\alpha) = \lim_{h\rightarrow 0} \frac{ (f\circ g)(\alpha+h)-(f\circ g)(\alpha) }{h}&=\lim_{h\rightarrow 0} \frac{ f(g(\alpha)+(Y-X)h)-(f(g(\alpha)) }{h} \\
    &= \ip{Y-X, \nabla f(g(\alpha))}\,.
\end{align*}
The proof for the case of the second derivative is analogous.
\end{proof}

We now extend the notion of Bregman divergence in the following. 
\begin{definition}
\label{def:complex bregman divergence}
 For any function $f$ that admits the gradient function $\nabla f$ at $X$, we define the Bregman divergence for any PD matrix $X$, and PSD matrix $Y$. 
 \begin{align*}
     D_f(Y,X) = f(Y)-f(X)-\ip{Y-X,\nabla f(X)}\,.
 \end{align*}
\end{definition}
The following lemma establishes the extension of the intermediate value theorem which holds over Hermitian matrices. 
\begin{lemma}[Intermediate value theorem]
\label{lemma: complex ivt}
    For any $f$ that admits the gradient function $\nabla f$ and Hessian function $\nabla^2 f$, and for all PD matrices $X$ and PSD matrices $Y$, there exists an $\alpha\in[0,1]$ and a matrix $\Xi := \alpha X+(1-\alpha)Y$, such that
    \begin{align*}
        D_f(Y, X) = \frac{1}{2}(\overrightarrow{Y}-\overrightarrow{X})^*\nabla^2 f(\Xi) (\overrightarrow{Y}-\overrightarrow{X})\,.
    \end{align*}
\end{lemma}
\begin{proof}
Define the function $g(t) = tY+(1-t)X$. Then consider the scalar function $f \circ g: [0,1] \rightarrow \bbR$. Then using the derivation in Lemma \ref{lemma: complex chain rule} we have that
\begin{align*}
    D_f(Y,X) = D_{f\circ g}(1,0)\,,
\end{align*}
where for the real function $f \circ g$ we use the standard definition of Bregman divergence which coincides with Definition \ref{def:complex bregman divergence}. Now by the intermediate value theorem for real valued functions, there exists $\alpha\in [0,1]$ such that
\begin{align*}
    D_{f\circ g}(1,0) &= \frac{1}{2}(1-0)^2 (f\circ g)''(\alpha)\,.
\end{align*}
Using Lemma \ref{lemma: complex chain rule} and combining the above we get the requisite conclusion.  \end{proof}

The following lemmas establish that the natural (derived via an extension of the real case) definition of gradient and Hessians are admissible for the loss functions of interest as well as the regularizer. 

\begin{lemma}
\label{lem: complex f grad}
For any PSD matrix $R$, and $PD$ matrix $X$, the function $f(X) = -\log(\ip{X,R})$ admits the gradient the function $\nabla f(X):= -\frac{R}{\ip{X,R}}$ and the Hessian function $\nabla^2 f(X):= \frac{\overrightarrow {R} \overrightarrow{ R}^*}{\ip{X,R}^2}$\,.
\end{lemma}
\begin{proof}
The requisite property for the gradient can be verified by the following calculations, 
\begin{align*}
    \lim_{h\rightarrow 0} \frac{-\log(\ip{X+h(Y-X),R})+\log(\ip{X,R})}{h}&=
    \lim_{h\rightarrow 0} \frac{-\log(1+h\frac{\ip{Y-X,R}}{\ip{X,R}})}{h}\\
    &=\lim_{h\rightarrow 0} \frac{-h\frac{\ip{Y-X,R}}{\ip{X,R}}+o(h^2)}{h}=-\frac{\ip{Y-X,R}}{\ip{X,R}}\,.
\end{align*}
Similarly for the Hessian, consider the following,
\begin{align*}
    &\lim_{h\rightarrow 0}h^{-1}\left(-\frac{\ip{Y-X,R}}{\ip{X+h(Y-X),R}}+\frac{\ip{Y-X,R}}{\ip{X,R}}\right)\\
    &=\lim_{h\rightarrow 0}h^{-1}\left(\frac{h\ip{Y-X,R}^2}{\ip{X+h(Y-X),R}\ip{X,R}}\right)\\
    &= \frac{\ip{Y-X,R}^2}{\ip{X,R}^2}=\frac{(\overrightarrow Y-\overrightarrow X)^*\overrightarrow R\overrightarrow R^*(\overrightarrow Y-\overrightarrow X)}{\ip{X,R}^2}\,.
\end{align*}
\end{proof}
\begin{lemma}
\label{lem: R function grad}
For any PSD matrix $R$, and $PD$ matrix $X$, the function $f(X) = -\log\det(X)$ admits the gradient function $\nabla f(X):=-X^{-1}$. Further given a PD matrix $X$, define $\nabla^2 f(X): \psds \in \Hermitian^{d^2}$ as the matrix satisfying the following for all Hermitian matrices $M$ \[\overrightarrow{M}^* \nabla^2 f(X)  \overrightarrow{M} =  \Tr(MX^{-1}MX^{-1})\,.\]
We have that the function $f(x)$ additionally admits the Hessian function $\nabla^2 f$. 
\end{lemma}
\begin{proof}
We use the following results from \citet{hjorungnes2007complex} (Table 2), which show that the differential along the Hermitian matrices are equal to real symmetric ones.
The differential of $\log\det(Z)$ is $\Tr(Z^{-1}\,\bm{d}Z)$ and the differential of $Z^{-1}$ is $-Z^{-1}\,\bm{d}ZZ^{-1}$\,.
In our case, the differential is $\bm{d}Z = (Y-X)$ and $Z$ is evaluated at $X$, hence
\begin{align*}
    \lim_{h\rightarrow 0} h^{-1}(-\log\det(X+h(Y-X))+\log\det(X)) = -\Tr((Y-X) X^{-1})\,,
\end{align*}
and 
\begin{align*}
    \lim_{h\rightarrow 0} -\Tr((Y-X) (X+h(Y-X))^{-1})+\Tr((Y-X) X^{-1})=\Tr((Y-X)X^{-1}(Y-X)X^{-1})\,.
\end{align*}
Finally, we show that there exists a Hermitian matrix $P\in\hessSet$ such that for all $M \in \Hermitian$ we have that, 
\begin{align*}
    \Tr(MX^{-1}MX^{-1}) = \overrightarrow{M}^* P \overrightarrow{M}\,.
\end{align*}
To this end note that
\begin{align*}
    \Tr(MX^{-1}MX^{-1})&=\Tr(X^{-\frac{1}{2}}MX^{-1}MX^{-\frac{1}{2}})\\
    &=\ve(X^{-\frac{1}{2}}MX^{-\frac{1}{2}})^*\ve(X^{-\frac{1}{2}}MX^{-\frac{1}{2}})\,.
\end{align*}
Since $X^{-\frac{1}{2}}MX^{-\frac{1}{2}}$ is linear in $M$, there exists a linear operator $N\in \hessSet$ such that \[\ve(X^{-\frac{1}{2}}MX^{-\frac{1}{2}})=N\overrightarrow{M}\,.\]
Therefore we have that
\begin{align*}
    \ve(X^{-\frac{1}{2}}MX^{-\frac{1}{2}})^*\ve(X^{-\frac{1}{2}}MX^{-\frac{1}{2}})=\overrightarrow{M}^* N^*N\overrightarrow{M}\,,
\end{align*}
which completes the argument defining $P = N^*N$.
\end{proof}

In the following lemma proves the Hessian of the $-\log \det()$ function is PD and lower bounded over the Hermitian subspace. 

\begin{lemma}
\label{lem: R Hessian lower bound}
For the function $f = -\log \det()$, the Hessian function $\nabla^2f$ satisfies the following properties. 
\begin{itemize}
    \item For any PD $X$, and any Hermitian $M \neq 0$, we have that
    \[ \overrightarrow{M}^* \nabla^2 f(X) \overrightarrow{M} > 0\]
    \item For any PD $X$, s.t. $\Tr(X) \leq 1$, and any Hermitian $M$, we have that
    \[ \overrightarrow{M}^* \nabla^2 f(X) \overrightarrow{M} \geq \|M\|^2\]
\end{itemize}
\end{lemma}

\begin{proof}
By definition in Lemma \ref{lem: R function grad} we have that 
\[\overrightarrow{M}^*\nabla^2 f(X)\overrightarrow{M} = \Tr(MX^{-1}MX^{-1}) = \Tr((X^{-1/2}MX^{-1/2})^2) = \sum_{i} \lambda_i^2(X^{-1/2}MX^{-1/2}) > 0,\] 
where the last inequality follows because we know that $M$ is Hermitian and not identically $0$. Next we show that for any Hermitian matrix $M$ and any PD matrix $X$ such that $\Tr(X) \leq 1$, we have that 
 \[ \overrightarrow{M}^*\nabla^2 f(X)\overrightarrow{M} \geq \|\overrightarrow{M}\|^2.\]
 This is proved as follows. By the spectral theorem, we can write $X^{-1}$ as $U^* \diag(\Lambda)U$, where $U$ is a unitary matrix, and $\Lambda$ is the vector of eigenvalues of $X^{-1}$, which are all at least $1$ since $\Tr(X) \leq 1$. Now by Lemma \ref{lem: R function grad} we have
 \begin{align*}
     \overrightarrow{M}^*\nabla^2 f(X)\overrightarrow{M} = \Tr(MX^{-1}MX^{-1}) &= \Tr(MU^* \diag(\Lambda)U M U^* \diag(\Lambda)U) \\ &= \Tr(\tilde{M} \diag(\Lambda)\tilde{M} \diag(\Lambda)),
 \end{align*}
 where $\tilde{M} := UMU^*$. Now consider the function $f: \mathbb{R}_+^d \rightarrow \mathbb{R}$ defined as
 \[f(\lambda) = \Tr(\tilde{M} \diag(\lambda)\tilde{M} \diag(\lambda)).\]
 An easy calculation using the fact that $\tilde{M}$ is Hermitian yields, for any $i \in [d]$, 
 \[\frac{\partial f (\lambda)}{\partial \lambda_i} = \sum_{k \neq i} |\tilde{M}_{ik}|^2 \lambda_k + 2|\tilde{M}_{ii}|^2\lambda_i \geq 0.\]
 Since $\Lambda \geq \mathbf{1}$ entrywise, the above inequality implies that
 \[\overrightarrow{M}^*\nabla^2 f(X)\overrightarrow{M} = f(\Lambda) \geq f(\mathbf{1}) = \Tr(\tilde{M}^2) = \|M\|^2.\]
\end{proof}

Further the following easy to verify lemma establishes the admissibility of gradients of the intermediate functions maintained by the algorithm. The definition of surrogate functions and the biased surrogate functions are naturally extended to the quantum learning case from the definitions provided in \eqref{eq: surrogate loss}, \eqref{eq: g loss}. 
\begin{lemma}
\label{lem: surrogate func grads}
The surrogate function $\hat{f}_t$ defined in \eqref{eq: surrogate loss} admits the following gradient and Hessian
\[ \nabla \hat f_t(X) := (1 +  \beta\ip{X-X_t,\nabla f_t(X_t)}) \nabla f_t(X_t), \]
\[ \nabla^2 \hat f_t(X) := \beta \overrightarrow{\nabla f_t(X_t)}\overrightarrow{\nabla f_t(X_t)}^*.\]
Further the biased surrogate function $g_{\tau}^e$ defined in \eqref{eq: g loss} admits the following gradient and Hessian
\[ \nabla g_{\tau}^e(X) := \nabla \hat f_{\tau}^e(X) + B\]
\[ \nabla^2 g_{\tau}^e(X) := \nabla^2 \hat f_{\tau}^e(X).\]
\end{lemma}
We now recall the definitions provided in the algorithm,
\[G_{\tau}^e(X) := \sum_{s = 1}^{\tau} g_{s}^e(X) + \eta^{-1}R(X) \;\;\text{ and }\;\; X_{\tau+1}^e := \argmin_{X \in \mathcal{A}}G_{\tau}^e(X).\]
\[\hat F_{\tau}^e(X) := \sum_{s = 1}^{\tau} \hat f_{s}^e(X) + \eta^{-1}R(X) \;\;\text{ and }\;\; U_{\tau+1}^e := \argmin_{X \in \mathcal{A}}\hat F_{\tau}^e(X).\]
Using Lemmas \ref{lem: R function grad},\ref{lem: surrogate func grads} we can analogously define admissible gradients and Hessian for both $G_{\tau}^e, \hat F_{\tau}^e$. We now have the following analogue for the minimality condition. 
\begin{lemma}
\label{lem: complex minima}
We have that the following statements hold for all $\tau, \epsilon$,
\begin{itemize}
    \item $X_{\tau+1}^e \succ 0, U_{\tau+1}^e \succ 0$, i.e. lie in the interior of the action set $\actionSet$
    \item Given any Hermitian matrices $X,U$ such that $\Tr(X) = \Tr(U) = 0$, we have that 
    \[ \ip{\nabla G_{\tau}^e(X_{\tau+1}^e), X} = 0 \text{ and } \ip{\nabla \hat F_{\tau}^e(U_{\tau+1}^e), U} = 0\]
\end{itemize}
\end{lemma}
\begin{proof}
The first statement is immediate by noting that for any $X \succeq 0$ with at least one eigenvalue approaching $0$, we have that $R(X)$ and thus $G_{\tau}^e$ approaches $\infty$ and for all PD matrices $\in \actionSet$, $G_{\tau}^e$ is finite. 

For the second statement we will prove the first inequality. The proof for the second inequality is analogous. We assume $X \neq 0$, otherwise the statement is immediate. Since $X_{\tau + 1} \succ 0$, there exists a $\delta > 0$ and a matrix $X^{+} = X_{\tau+1}^e + \delta X$ such that $X^{+} \in \actionSet$. Consider the function $X(\alpha) := \alpha X_{\tau+1}^{e} + (1-\alpha)X^+$ over $\alpha \in [0,1]$. We have that for all $\alpha$ there exists some $\alpha' \in [0, \alpha]$ such that the following holds
\begin{align*}
    G_{\tau}^e(X(\alpha)) &- G_{\tau}^e(X_{\tau+1}^e) = \alpha \ip{\nabla G_{\tau}^e(X_{\tau+1}^e), X^+ - X_{\tau+1}^e} + D_{G_{\tau}^e}(X(\alpha), X_{\tau+1}^e) \\
    &= \alpha \ip{\nabla G_{\tau}^e(X_{\tau+1}^e), X^+ - X_{\tau+1}^e} + \frac{\alpha^2}{2}(\overrightarrow{X}^+-\overrightarrow{X_{\tau+1}^e})^*\nabla^2 G_{\tau}^e(X(\alpha')) (\overrightarrow{X}^+-\overrightarrow{X_{\tau+1}^e}).
\end{align*}
Since $X_{\tau+1}^e$ is the minimizer we have that for all $\alpha$, 
\[\alpha \ip{\nabla G_{\tau}^e(X_{\tau+1}^e), X^+ - X_{\tau+1}^e} + \frac{\alpha^2}{2}(\overrightarrow{X}^+-\overrightarrow{X_{\tau+1}^e})^*\nabla^2 G_{\tau}^e(X(\alpha')) (\overrightarrow{X}^+-\overrightarrow{X_{\tau+1}^e}) \geq 0\]
Using very coarse bounds obtained through combining Lemmas \ref{lem: surrogate func grads}, \ref{lem: complex f grad} and \ref{lem: R function grad}  it is easy to see that there exists a finite number $L$ independent of $\alpha$ (but potentially dependent on other problem parameters like $T$), such that $\nabla^2 G_{\tau}^e(X(\alpha')) \preceq L \cdot \bbI_{d^2}$. This further implies that for all $\alpha \in [0,1]$ we have that 
\[ \alpha \ip{\nabla G_{\tau}^e(X_{\tau+1}^e), X^+ - X_{\tau+1}^e} + \frac{Ld^2\alpha^2}{2} \geq 0.\]
Now in case $\ip{\nabla G_{\tau}^e(X_{\tau+1}^e), X^+ - X_{\tau+1}^e} < 0$, then we can set $\alpha$ to be appropriately small such that the above expression is strictly negative which is a contradiction. Hence we have that 
\[\ip{\nabla G_{\tau}^e(X_{\tau+1}^e), X^+ - X_{\tau+1}^e} \geq 0 \Rightarrow \ip{\nabla G_{\tau}^e(X_{\tau+1}^e), X} \geq 0.\] 
Repeating the argument by replacing $X$ with $-X$ gives that $\ip{\nabla G_{\tau}^e(X_{\tau+1}^e), X} \leq 0$ and thus $\ip{\nabla G_{\tau}^e(X_{\tau+1}^e), X} = 0$. 
\end{proof}

We provide the proof of the following lemma whose restriction over the reals is well-known and is used repeatedly in the proofs of Online Newton Step like algorithms. 

\begin{lemma}
\label{lem: telescoping sum}
Given a sequence of PD matrices $X_1 \preceq X_2 \preceq X_3 \ldots X_T$, we have that
\[ \sum_{t=1}^{T-1} \ip{X_{t+1}^{-1}, X_{t+1}-X_t} \leq \log(\det(X_T)) - \log(\det(X_1))\]
\end{lemma}

\begin{proof}
We first begin by providing the proof of a simpler statement which implies the above statement via a simple summation. Given two PD matrices $X \preceq Y$, we have that 
\[ \ip{Y^{-1}, Y-X} \leq \log(\det(Y)) - \log(\det(X))\]
To prove the above we consider the following function $\phi(\alpha)$ defined as
\[ \phi(\alpha) := \log \det(\alpha Y + (1-\alpha) X)\]
Using Lemma \ref{lemma: complex chain rule} and the calculations in Lemma \ref{lem: R function grad} we see that $\phi$ is a concave function over $\alpha$ and that 
\[ \phi'(\alpha) = \ip{Y - X, (\alpha Y + (1-\alpha) X)^{-1}}.\]
Therefore using concavity we have that $\phi'(1) \leq \phi(1) - \phi(0)$ which implies the requisite statement by substitution.

\end{proof}

\section{Algorithm for Quantum Learning with Log Loss}
\label{sec: app other prelims}
\begin{algorithm2e}[H]
\setstretch{1.25}
    \textbf{input}: $T$, $B$, $\eta$, $\beta$. \\
    \textbf{initialize}:  $\forall e\in\mathbb{N}:\,P^e_0=d\identity, G^e_0(\cdot) = \hat{F}^e_0(\cdot) = \eta^{-1}R(\cdot), X^e_1=U^e_1=\argmin_{X\in\actionSet}G_0^e(X)$.\\
    $e\leftarrow 1, \tau\leftarrow 1$\\
    \For{$t=1,\dots$}{
        $f_t\leftarrow$ receive from playing $X_t\leftarrow X^e_{\tau}$.\\
        $\hat f^e_\tau=\hat f_t\leftarrow$ construct according to \eqref{eq: surrogate loss}.\\
        $\hat F^e_\tau \leftarrow \hat F^e_{\tau-1} + \hat f^e_\tau$\\
        $G^e_\tau  \leftarrow G^e_{\tau-1} + g^e_\tau$, where $g^e_\tau(X) := \hat f^e_\tau(X) - \ip{X, P^e_\tau - P^e_{\tau-1}}B$.\\
        
        $X^e_{\tau+1} \leftarrow  \argmin_{X\in\actionSet} G^e_{\tau}(X)$, $U^e_{\tau+1} \leftarrow \argmin_{X\in\actionSet} \hat{F}^e_{\tau}(X)$ \\
        $P_{\tau+1}^e =P_{\tau}^e+ [X_{\tau+1}^e]^{-\frac{1}{2}}\left(\identity-[X_{\tau+1}^e]^{\frac{1}{2}}P_{\tau}^e [X_{\tau+1}^e]^{\frac{1}{2}}\right)_+[X_{\tau+1}^e]^{-\frac{1}{2}}$ \\
        \uIf{$U_t \not\prec \frac{1}{2(1+6\eta)\beta} P_t^{-1} $}{
              $e\leftarrow e+1,
            \tau\leftarrow 1$ \tcp{Reset the algorithm}
        }\Else{
        $\tau \leftarrow \tau+1$
        }
    }
    \caption{\qbisons}\label{alg: bisons (psd)}
\end{algorithm2e}

\section{Preliminary definitions and properties}
\label{sec:app complex definitions}
In this section we provide some general definitions and other properties necessary for the analysis of the BISONS algorithm. Given a PSD matrix $A\in\psds$, we associate a norm over $\psds$, defined for any $W \in \psds$ as
\[\norm{W}_A = \sqrt{\Tr(WAWA)}= \sqrt{\Tr((A^{1/2}WA^{1/2})^2)}\,.\]

\begin{lemma}
For any positive semi-definite $A$, $\norm{\cdot}_A$ is a pseudo-norm.
\end{lemma}
\begin{proof}
 The only non-trivial property is the triangle inequality.
 We have
 \begin{align*}
     \norm{V+W}_A^2&=\Tr((A^\frac{1}{2}(V+W)A^{\frac{1}{2}})^2)= \norm{V}_A^2+\norm{W}_A^2+2\Tr((A^\frac{1}{2}VA^\frac{1}{2})(A^\frac{1}{2}WA^\frac{1}{2}) ) \\
     &\leq \norm{V}_A^2+\norm{W}_A^2+2\sqrt{\Tr((A^\frac{1}{2}VA^\frac{1}{2})^2)\Tr((A^\frac{1}{2}WA^\frac{1}{2})^2 )}=(\norm{V}_A+\norm{W}_A)^2\,,
 \end{align*}
 where the inequality is due to the fact that for PSD matrices $A,B$, we have that $\Tr(AB)\leq \sqrt{\Tr(A^2)\Tr(B^2)}$, which follows from the Cauchy-Schwarz inequality.
\end{proof}

\begin{lemma}
\label{lem: mat norm to ratio}
    For any PD matrices $A,B$ such that
    \[\norm{A-B}_{B^{-1}} \leq \lambda\]
    for some $\lambda \geq 0$. Then, it holds that the eigenvalues of $B^{-\frac{1}{2}}AB^{-\frac{1}{2}}$ lie within the interval $[1-\lambda,1+\lambda]$.
\end{lemma}
\begin{proof}
    We have
    \begin{align*}
        \norm{A-B}_{B^{-1}}^2&=\Tr((A-B)B^{-1}(A-B)B^{-1})\\
        &=\Tr((B^{-\frac{1}{2}}AB^{-\frac{1}{2}}-I)^2) \\
        &= \sum_{i=1}^d \left(\operatorname{ev}_i(B^{-\frac{1}{2}}AB^{-\frac{1}{2}})-1\right)^2\leq \lambda^2,
    \end{align*}
    where $\operatorname{ev}_i$ represents the $i^{th}$ eigenvalue. Therefore 
    every eigenvalue satisfies \[
    \lvert\operatorname{ev}_i(B^{-\frac{1}{2}}AB^{-\frac{1}{2}})-1\rvert \leq \lambda\,.\]
\end{proof}
For any PSD matrix $A, B$ define 
\begin{equation}
\label{eq:abintervaldef}
    [A,B]:=\{\alpha A+(1-\alpha)B\,\vert\, \alpha\in[0,1]\}.
\end{equation}
\begin{lemma}
\label{lem: mat norm to ratio 2}
For any PD matrix $A,B$ such that
\[
\max_{C\in[A,B]} \norm{A-B}_{C^{-1}} \leq \lambda\,,
\]
for some $\lambda \geq 0$. Then it holds for all $D,E\in[A,B]$:
\[D\preceq (1+\lambda)E\,,\quad D^{-1}\preceq (1+\lambda)E^{-1}\,.\]
\end{lemma}
\begin{proof}
Since $D,E\in[A,B]$, there exists $c:\, |c|\leq 1$ such that $D-E=c(A-B)$. Hence 
\[ \norm{D-E}_{D^{-1}}\leq \lambda\,.\]
Applying Lemma~\ref{lem: mat norm to ratio} completes the first part.
Repeating the same argument, but now starting with $ \norm{D-E}_{E^{-1}}\leq \lambda$,
yields the second claim.
\end{proof}

\section{\bisons detailed analysis}
\label{sec:bisons-analysis}

In this section we provide the details for the analysis of our algorithms \ref{alg: bisons}, \ref{alg: bisons (psd)}, eventually proving Theorems \ref{thm: single epoch} and \ref{thm: single epoch psd}. Before delving into the analysis we request the reader to familiarize themselves with the requisite notation, definition and properties listed out in Sections \ref{sec:app complex definitions}, \ref{sec: app other prelims}. Since \bisons is a special case of \qbisons, we will provide the analysis focused on the quantum learning case, i.e. the domain will be PSD Hermitian matrices, however all the statements will hold when these matrices are real and diagonal as will be the case for the online optimal portfolio.

We first provide a proof of Lemma \ref{lem: lower surrogate}. We further begin the core analysis by providing some useful auxiliary lemmas and the lemmas governing the stability of the output of the algorithm in the next two subsections. We will restrict attention in the next two subsections to any fixed epoch and there for brevity we will remove the epoch superscript $e$, from the lemma statements as well as proofs. All the statements should be understood to hold for any particular epoch. 

\subsection{Proof of Lemma \ref{lem: lower surrogate}}
\begin{proof}
Equality at $x=y_t$ holds by construction. We have $h'(x)=-x^{-1}$, which is concave and $\underline{\hat h'_t}(x) = \min\{-(1+\beta){y^{-1}_t}+ \beta x y_{t}^{-2}, -\beta y_t^{-1}\}$, which is piece-wise linear. A quick calculation shows $h'(y_t)=\underline {h'_t}(y_t)$ and $h'(\beta^{-1}y_t)=\underline {h'_t}(\beta^{-1}y_t)$. Hence
for $x<y_t$, we have $h'(x) < \underline {h'_t}(x)$ and for $\beta^{-1}y_t\geq x>y_t$, we have $h'(x) > \underline {h'_t}(x)$. Finally, the derivative of $h'(x)$ is monotonically increasing which implies $h'(x) >\underline {h'_t}(x)$ for $x>\beta^{-1}y_t$, which completes the proof.
\end{proof}

\subsection{Auxiliary Lemmas}
In this section we collect some basic lemmas regarding the matrices $X_{\tau}, P_{\tau}$ generated by the algorithm. We recall the definition of $P_{\tau}$ defined in \eqref{eq: P update} as
\begin{align*}
    P_{\tau} :=P_{\tau-1}+ X_{\tau}^{-\frac{1}{2}}\left(\identity-X_{\tau}^{\frac{1}{2}}P_{\tau-1}X_{\tau}^{\frac{1}{2}}\right)_+X_{\tau}^{-\frac{1}{2}}\,,
\end{align*}
which in particular implies that for all $\tau$,
\[\sqrtX_{\tau+1}(P_{\tau+1}-P_{\tau})\sqrtX_{\tau+1} =\left(\identity - \sqrtX_{\tau+1}P_{\tau}\sqrtX_{\tau+1}\right)_+\,.\]
The next two lemmas state the main properties satisfied by our choice of $P_{\tau}$. These properties prompt the choice of the definition for $P_{\tau}$. 
\begin{lemma}
\label{lemma:ptlowerbound}
    We have that for all $\tau$,
    \[P_{\tau} \succeq P_{\tau-1} \quad \text{ and } \quad P_{\tau} \succeq X_{\tau}^{-1}.\]
\end{lemma}
\begin{proof}
The first statement is immediate from the definition of $P_{\tau}$. For the second inequality note that 
\begin{align}
    \sqrtX_{\tau}P_{\tau}\sqrtX_{\tau} = \sqrtX_{\tau}P_{\tau-1}\sqrtX_{\tau} + \left(\identity - \sqrtX_{\tau}P_{\tau-1}\sqrtX_{\tau}\right)_+ \succeq \identity,
\end{align}
which implies that $P_{\tau} \succeq X_{\tau}^{-1}$. 
\end{proof}

\begin{lemma}
\label{lem: sum of P}
    For any $\tau$, we have
    \begin{align*}
        \ip{X_{\tau+1}, P_{\tau+1}-P_{\tau}}=\ip{P_{\tau+1}^{-1}, P_{\tau+1}-P_{\tau}}
    \end{align*}
\end{lemma}

\begin{proof}
Recall, by definition
\begin{align*}
    P_{\tau+1}=P_{\tau} + X_{\tau+1}^{-\frac{1}{2}}\left(\identity-X_{\tau+1}^{\frac{1}{2}}P_{\tau}X_{\tau+1}^{\frac{1}{2}}\right)_{+}X_{\tau+1}^{-\frac{1}{2}}\,.
\end{align*}
Hence
\begin{align*}
    \ip{X_{\tau+1}, P_{\tau+1}-P_{\tau}} = \Tr\left(\left(\identity-X_{\tau+1}^{\frac{1}{2}}P_{\tau}X_{\tau+1}^{\frac{1}{2}}\right)_{+}\right)= \sum_{i=1}^d \max\{1-\lambda_i,0\}\,,
\end{align*}
where $\lambda_i$ are the eigenvalues of $X_{\tau+1}^{\frac{1}{2}}P_{\tau}X_{\tau+1}^{\frac{1}{2}}$.
For the RHS, we have
\begin{align*}
    \ip{P_{\tau+1}^{-1}, P_{\tau+1}-P_{\tau}}&=\ip{(X_{\tau+1}^\frac{1}{2}P_{\tau+1}X_{\tau+1}^\frac{1}{2})^{-1}, X_{\tau+1}^\frac{1}{2}(P_{\tau+1}-P_{\tau})X_{\tau+1}^\frac{1}{2}}\,.
\end{align*}
Note that $(X_{\tau+1}^\frac{1}{2}P_{\tau+1}X_{\tau+1}^\frac{1}{2})=X_{\tau+1}^{\frac{1}{2}}P_{\tau}X_{\tau+1}^{\frac{1}{2}}+\left(\identity-X_{\tau+1}^{\frac{1}{2}}P_{\tau}X_{\tau+1}^{\frac{1}{2}}\right)_{+}$ modifies the eigenvalues of $X_{\tau+1}^{\frac{1}{2}}P_{\tau}X_{\tau+1}^{\frac{1}{2}}$ such that they are lower bounded by 1.
Therefore
\begin{align*}
    \ip{(X_{\tau+1}^\frac{1}{2}P_{\tau+1}X_{\tau+1}^\frac{1}{2})^{-1}, X_{\tau+1}^\frac{1}{2}(P_{\tau+1}-P_{\tau})X_{\tau+1}^\frac{1}{2}}=\sum_{i=1}^d\frac{\max\{1-\lambda_i,0\}}{\max\{1,\lambda_i\}}=\sum_{i=1}^d\max\{1-\lambda_i,0\}\,,
\end{align*}
where the last equality follows from the nominator being non-zero only if the denominator is 1.
\end{proof}

The following is a useful lemma we collect here.
\begin{lemma}
\label{lem: p increase by x}
For any $\tau$, it holds that
\[
\norm{P_{\tau+1}-P_{\tau}}_{X_{\tau+1}}\leq \norm{X_{\tau+1}-X_{\tau}}_{X_{\tau}^{-1}}\,.
\]
\end{lemma}
\begin{proof}
Denote $\tilde D= X_{\tau+1}^\frac{1}{2}(P_{\tau+1}-P_{\tau})X_{\tau+1}^\frac{1}{2}$, therefore we have that
\[\tilde D = \left(\identity-X_{\tau+1}^{\frac{1}{2}}P_{\tau}X_{\tau+1}^{\frac{1}{2}}\right)_+\,.\]
We have
\begin{align*}
\norm{P_{\tau+1}-P_{\tau}}_{X_{\tau+1}}^2 = \Tr(\tilde D^2)&=\Tr\left((\identity-X_{\tau+1}^{\frac{1}{2}}P_{\tau}X_{\tau+1}^{\frac{1}{2}})_+^2\right)\\
&\leq \Tr\left((\identity-X_{\tau+1}^{\frac{1}{2}}X_{\tau}^{-1}X_{\tau+1}^{\frac{1}{2}})_+^2\right)\\
&\leq \Tr\left((\identity-X_{\tau+1}^{\frac{1}{2}}X_{\tau}^{-1}X_{\tau+1}^{\frac{1}{2}})^2\right)\\
&= \norm{X_{\tau+1}-X_{\tau}}^2_{X_{\tau}^{-1}}\,,
\end{align*}
where the first inequality uses the fact $P_{\tau} \succeq X_{\tau}^{-1}$ from Lemma \ref{lemma:ptlowerbound}.
\end{proof}
Finally as a result of our reset condition we have the following lemma. 

\begin{lemma}
    \label{lem: U bound}
    Let $s \leq \tau$ be two time indices belonging to the same epoch, such that the reset condition was not triggered upto time index $\tau-1$. Then we have that 
    \[U_{\tau} \preceq \beta^{-1}X_s\]
\end{lemma}

\begin{proof}
Due to the reset condition, we know that $U_{\tau} \preceq (2(1+6\eta)\beta)^{-1}P_{\tau}^{-1} \preceq \beta^{-1}P_{\tau}^{-1}$. Further since by Lemma \ref{lemma:ptlowerbound}, $P_{\tau}^{-1} \preceq X_s$, the lemma follows. 
\end{proof}

\subsection{Stability Lemmas}
In this section we show that successive iterates $X_{\tau}, U_{\tau}, P_{\tau}$ and $X_{\tau+1}, U_{\tau+1}, P_{\tau+1}$ do not move too far away from each other due to the log barrier, establishing the requisite stability of our method. These results are summarized in lemmas \ref{lem: X stability} and \ref{lem: U stability}. Our stability lemmas hold under the following constraints over the algorithm parameters $\eta, \beta$. 
\begin{align}
    \eta &\leq \min\{\frac{1}{4B}, \frac{\beta}{4},\frac{1}{63}\} \label{eq: eta constraints}\\
    \beta &\leq \sqrt{2}-1 \label{eq: beta constraints}\\
    T &\geq \max\{2d, \beta^{-1}\} \label{eq: T constraints}
\end{align}

\begin{lemma}
\label{lem: X stability}
If $\eta$ satisfies constraint~\eqref{eq: eta constraints}, then for any $t\in[T]$:
\begin{align*}
    \frac{1}{1+6\eta}X_{\tau+1}\preceq X_{\tau} \preceq (1+6\eta)X_{\tau+1}\,,
\end{align*}
as well as
\begin{align*}
    P_{\tau+1}\preceq (1+6\eta)P_{\tau}\,.
\end{align*}
\end{lemma}

\begin{lemma}
\label{lem: U stability}
If $\eta$ and $\beta$ satisfy constraints~\eqref{eq: eta constraints} and \eqref{eq: beta constraints} and no reset is triggered at time $\tau-1$, then \[U_{\tau+1}\preceq 2U_{\tau}.\]
\end{lemma}
We first present some auxiliary lemmas and then prove the above stability lemmas. 
\begin{lemma}
\label{lem: P stability}
For any $\tau$, it holds that
\begin{align*}
    &X_{\tau+1}^{-1} \preceq (1+\lambda) X_{\tau}^{-1}\quad \Rightarrow\quad P_{\tau+1}\preceq (1+\lambda) P_{\tau}\,.
\end{align*}
\end{lemma}
\begin{proof}
We assume LHS above is true. By Lemma \ref{lemma:ptlowerbound} we have that,
\begin{align*}
X_{\tau+1}^{-1} \preceq (1+\lambda) X_{\tau}^{-1}\preceq (1+\lambda) P_{\tau}\,.     
\end{align*}
Therefore we have that, 
\begin{align*}
    \identity \preceq(1+\lambda)\sqrtX_{\tau+1}P_{\tau}\sqrtX_{\tau+1}\,.
\end{align*}
Hence by definition of $P_{\tau+1}$,
\begin{align*}
    \sqrtX_{\tau+1}(P_{\tau+1} - (1+\lambda) P_{\tau})\sqrtX_{\tau+1}&=(\identity - \sqrtX_{\tau+1}P_{\tau}\sqrtX_{\tau+1})_+-\lambda\sqrtX_{\tau+1}P_{\tau}\sqrtX_{\tau+1}\\
    &\preceq (\identity - (1+\lambda)\sqrtX_{\tau+1}P_{\tau}\sqrtX_{\tau+1})_+=0\,.
\end{align*}
Finally, this implies
\begin{align*}
    (P_{\tau+1} - (1+\lambda) P_{\tau})\preceq 0 \quad \Leftrightarrow\quad P_{\tau+1}\preceq (1+\lambda) P_{\tau}
\end{align*}
as claimed.
\end{proof}

\begin{lemma}
\label{lem: x stability}
    If $\eta$ satisfies constraint~\eqref{eq: eta constraints},  then for all $\tau$ we have that,
    \begin{align*}
       \max_{\Xi\in[X_{\tau},X_{\tau+1}]}\norm{X_{\tau+1}-X_{\tau}}_{\Xi^{-1}} \leq 6\eta
    \end{align*}
\begin{align*}
\end{align*}    
\end{lemma}
\begin{proof}
The proof follows by induction.  Set by convention $X_0=X_1$, then the condition holds for $\tau=0$.
Now assuming that the condition holds for all time-steps including $\tau-1$, we prove in the following that this holds for $\tau$. We will show this by contradiction. To this end suppose that $\max_{\Xi\in[X_{\tau},X_{\tau+1}]}\norm{X_{\tau+1}-X_{\tau}}_{\Xi^{-1}} > 6\eta $. By continuity there exists a point $X \in (X_{\tau}, X_{\tau+1})$ such that $\max_{\Xi\in[X_{\tau},X]}\norm{X-X_{\tau}}_{\Xi^{-1}} = 6\eta $. Now recall the definitions, 
\[G_{\tau}(X) := \sum_{s = 1}^{\tau} g_{s}(X) + \eta^{-1}R(X) \;\;\text{ and }\;\; X_{\tau+1} := \argmin_{X \in \mathcal{A}}G_{\tau}(X).\]
In the latter half of the proof will show that the condition on $X$ implies that $G_{\tau}(X) \geq G_{\tau}(X_{\tau})$. We will first show why establishing the above leads to a contradiction. So we assume $G_{\tau}(X) \geq G_{\tau}(X_{\tau})$. To this end consider the scalar function $\phi(\alpha)$ for $\alpha \in [0,1]$ defined as
\[ \phi(\alpha) = G_{\tau}(\alpha X_{\tau + 1} + (1-\alpha)X_{\tau}).\]
Let $\alpha_X \in (0,1)$ correspond to the unique $\alpha$ such that 
\[ X = \alpha X_{\tau+1} + (1-\alpha) X_{\tau}.\]
The assumption $G_{\tau}(X) \geq G_{\tau}(X_{\tau})$ implies that $\phi(\alpha_X) \geq \phi(0)$. Further since $X_{\tau+1}$ is the minimizer, we have that $\phi(1) \leq \phi(0)$. We will now show that $\phi$ is a strictly convex function and that will contradict the above derived statements. To show that $\phi$ is strictly convex we use Lemma \ref{lemma: complex chain rule} to establish that for any $\alpha$ 
\[ \phi''(\alpha) = (\overrightarrow{X}_{\tau+1}-\overrightarrow{X}_{\tau})^*\nabla^2 G_{\tau}(\alpha X_{\tau+1} + (1-\alpha) X_{\tau})) (\overrightarrow{X}_{\tau+1}-\overrightarrow{X}_{\tau}).\]
Further Lemma \ref{lem: surrogate func grads}  establishes that for any $Y$, $\nabla^2 G_{\tau}(Y) \succeq \eta^{-1}R(Y)$. Using the calculation of $\nabla^2 R$ in Lemma \ref{lem: R function grad} and noting that $X_{\tau}, X_{\tau+1}$ are positive-definite, we get that
\[ \phi''(\alpha) = (\overrightarrow{X}_{\tau+1}-\overrightarrow{X}_{\tau})^*\nabla^2 G_{\tau}(\alpha X_{\tau+1} + (1-\alpha) X_{\tau})) (\overrightarrow{X}_{\tau+1}-\overrightarrow{X}_{\tau}) > 0,\]
which establishes the strict convexity of $\phi$ and therefore the contradiction. 

All that is left to show now is that $G_{\tau}(X) \geq G_{\tau}(X_{\tau})$. To this end consider the following. Using gradient and Hessian via Lemma \ref{lem: surrogate func grads} and via the intermediate value lemma (Lemma \ref{lemma: complex ivt}), we have that there exists $\Psi\in[X_{\tau},X]$ such that
\[G_{\tau}(X) = G_{\tau}(X_{\tau}) + \ip{\nabla G_{\tau}(X_{\tau}), X-X_{\tau}}+\frac{1}{2} (\overrightarrow{X}-\overrightarrow{X}_{\tau})^* \nabla^2 G_{\tau}(\Psi)(\overrightarrow{X}-\overrightarrow{X}_{\tau}).\]
Further noting that fact that $\nabla G_{\tau}(\cdot) - \nabla g_{\tau}(\cdot) = \nabla G_{\tau-1}(\cdot)$, $X_{\tau} := \argmin_{X \in \mathcal{A}} G_{\tau-1}(X)$ and using Lemma \ref{lem: complex minima} we get that 
\[G_{\tau}(X) \geq G_{\tau}(X_{\tau}) + \ip{\nabla g_{\tau}(X_{\tau}), X-X_{\tau}}+\frac{1}{2} (\overrightarrow{X}-\overrightarrow{X}_{\tau})^* \nabla^2 G_{\tau}(\Psi)(\overrightarrow{X}-\overrightarrow{X}_{\tau}).\]
Further using $\nabla^2 G_{\tau}(\cdot) \succeq \eta^{-1}\nabla^2 R(\cdot)$ and considering the computation of $\nabla^2 R$ presented in Lemma \ref{lem: R function grad} and that for any  $\Xi,\Phi\in[X_{\tau},X_{\tau+1}]$ we have that $\Xi^{-1}\preceq (1+6\eta)\Phi^{-1}$ according to Lemma~\ref{lem: mat norm to ratio 2}, we get that
\begin{align*}
    G_{\tau}(X) &\geq G_{\tau}(X_{\tau}) + \ip{\nabla g_{\tau}(X_{\tau}), X-X_{\tau}}+\frac{1}{2} (\overrightarrow{X}-\overrightarrow{X}_{\tau})^* \nabla^2 G_{\tau}(\Psi)(\overrightarrow{X}-\overrightarrow{X}_{\tau})\\
    &\geq G_{\tau}(X_{\tau}) - \norm{\nabla g_{\tau}(X_{\tau})}_{X_{\tau}}\norm{X-X_{\tau}}_{X_{\tau}^{-1}}+\frac{1}{2\eta(1+6\eta)^2}\max_{\Xi\in[X_{\tau},X_{\tau+1}]} \norm{X - X_{\tau}}_{\Xi^{-1}}\\
    &= G_{\tau}(X_{\tau}) -6\eta \norm{\nabla g_{\tau}(X_{\tau})}_{X_{\tau}}+\frac{18\eta}{(1+6\eta)^2}\,.
\end{align*}
We now show that $\norm{\nabla g_{\tau}(X_{\tau})}_{X_{\tau}}\leq \frac{3}{(1+6\eta)^2}$ which completes the inductive step. 
We have 
\begin{align*}
    \norm{\nabla g_{\tau}(X_{\tau})}_{X_{\tau}} &\leq \norm{\nabla \hat f_{\tau}(X_{\tau})}_{X_{\tau}}+  B\norm{P_{\tau}-P_{\tau-1}}_{X_{\tau}}\\
    &= \frac{\sqrt{\Tr( (X_{\tau}^\frac{1}{2}R_{\tau}X_{\tau}^\frac{1}{2})^2)}}{\Tr(X_{\tau}^\frac{1}{2}R_{\tau}X_{\tau}^\frac{1}{2})}+B\norm{P_{\tau}-P_{\tau-1}}_{X_{\tau}}\\
    &\leq 1+\norm{X_{\tau}-X_{\tau-1}}_{X_{\tau}^{-1}}B\\
    &\leq 1+6\eta B\leq \frac{5}{2} = \frac{3}{(1+(\sqrt{6/5}-1) )^2}\leq \frac{3}{(1+6\eta)^2},
\end{align*}
where the last set of inequalities follow from the constraints on $\eta$ defined in \eqref{eq: eta constraints} and induction assumption. 
\end{proof}
We now prove the stability lemmas, Lemma \ref{lem: X stability} and Lemma \ref{lem: U stability}.  

\begin{proof}[Proof of Lemma \ref{lem: X stability}]
Combining Lemma~\ref{lem: x stability} and \ref{lem: mat norm to ratio} yields the claim for $X$. By Lemma~\ref{lem: P stability} this follows for $P$ as well.
\end{proof}

\begin{proof}[Proof of Lemma \ref{lem: U stability}]
We will show that
\begin{align*}
    \norm{U_{\tau+1}-U_{\tau}}_{U_{\tau}^{-1}}\leq 1\,,
\end{align*}
which by Lemma~\ref{lem: mat norm to ratio} implies the statement of the lemma. To show the above we use a similar proof structure as in the case of Lemma \ref{lem: x stability} and assume for contradiction that $\norm{U_{\tau+1}-U_{\tau}}_{U_{\tau}^{-1}} > 1$. Once again by continuity we have that there exists a point $U \in [U_{\tau}, U_{\tau+1}]$ such that $\norm{U-U_{\tau}}_{U_{\tau}^{-1}} = 1$.
Now recall the definitions, 
\[\hat F_{\tau}(X) := \sum_{s = 1}^{\tau} \hat f_{s}(X) + \eta^{-1}R(X) \;\;\text{ and }\;\; U_{\tau+1} := \argmin_{X \in \mathcal{A}}\hat F_{\tau}(X).\]
In the latter half of the proof will show that the condition on $U$ implies that $\hat F_{\tau}(X) \geq \hat F_{\tau}(X_{\tau})$. We will first show why establishing the above leads to a contradiction. So we assume $\hat F_{\tau}(X) \geq \hat F_{\tau}(X_{\tau})$. To this end consider the scalar function $\phi(\alpha)$ for $\alpha \in [0,1]$ defined as
\[ \phi(\alpha) = \hat F_{\tau}(\alpha U_{\tau + 1} + (1-\alpha)U_{\tau}).\]
Let $\alpha_U \in (0,1)$ correspond to the unique $\alpha$ such that 
\[ U = \alpha U_{\tau+1} + (1-\alpha) U_{\tau}.\]
The assumption $\hat F_{\tau}(U) \geq \hat F_{\tau}(U_{\tau})$ implies that $\phi(\alpha_U) \geq \phi(0)$. Further since $U_{\tau+1}$ is the minimizer, we have that $\phi(1) \leq \phi(0)$. We will now show that $\phi$ is a strictly convex function and that will contradict the above derived statements. To show that $\phi$ is strictly convex we use Lemma \ref{lemma: complex chain rule} to establish that for any $\alpha$ 
\[ \phi''(\alpha) = (\overrightarrow{U}_{\tau+1}-\overrightarrow{U}_{\tau})^*\nabla^2 \hat F_{\tau}(\alpha U_{\tau+1} + (1-\alpha) U_{\tau})) (\overrightarrow{U}_{\tau+1}-\overrightarrow{U}_{\tau}).\]
Further Lemma \ref{lem: surrogate func grads}  establishes that for any $Y$, $\nabla^2 \hat F_{\tau}(Y) \succeq \eta^{-1}R(Y)$. Using the calculation of $\nabla^2 R$ in Lemma \ref{lem: R function grad} and noting that $U_{\tau}, U_{\tau+1}$ are positive-definite, we get that
\[ \phi''(\alpha) = (\overrightarrow{U}_{\tau+1}-\overrightarrow{U}_{\tau})^*\nabla^2 \hat F_{\tau}(\alpha U_{\tau+1} + (1-\alpha) U_{\tau})) (\overrightarrow{U}_{\tau+1}-\overrightarrow{U}_{\tau}) > 0,\]
which establishes the strict convexity of $\phi$ and therefore the contradiction. 

Therefore all we need to establish is that $\hat F_{\tau}(U)\geq \hat F_{\tau}(U_{\tau})$. To this end consider the following. Using gradient and Hessian via Lemma \ref{lem: surrogate func grads} and via the intermediate value lemma (Lemma \ref{lemma: complex ivt}), we have that there exists $\Xi\in[U,U_{\tau}]$ such that
\[\hat F_{\tau}(U) = \hat F_{\tau}(U_{\tau}) + \ip{\nabla \hat F_{\tau}(U_{\tau}), U-U_{\tau}}+\frac{1}{2} (\overrightarrow{U}-\overrightarrow{U}_{\tau})^* \nabla^2 \hat F_{\tau}(\Xi)(\overrightarrow{U}-\overrightarrow{U}_{\tau}).\]
Further noting that fact that $\nabla \hat F_{\tau}(\cdot) - \nabla \hat f_{\tau}(\cdot) = \nabla \hat F_{\tau-1}(\cdot)$, $U_{\tau} := \argmin_{U \in \mathcal{A}} \hat F_{\tau-1}(X)$ and using Lemma \ref{lem: complex minima} we get that 
\[\hat F_{\tau}(U) \geq \hat F_{\tau}(U_{\tau}) + \ip{\nabla \hat f_{\tau}(U_{\tau}), U-U_{\tau}}+\frac{1}{2} (\overrightarrow{U}-\overrightarrow{U}_{\tau})^* \nabla^2 \hat F_{\tau}(\Xi)(\overrightarrow{U}-\overrightarrow{U}_{\tau}).\]
Further using $\nabla^2 \hat F_{\tau}(\cdot) \succeq \eta^{-1}\nabla^2 R(\cdot)$ and considering the computation of $\nabla^2 R$ presented in Lemma \ref{lem: R function grad} we have that,
\begin{align*}
    \hat F_{\tau}(U) 
    &\geq \hat F_{\tau}(U_{\tau}) - \norm{\nabla \hat f_{\tau}(U_{\tau})}_{U_{\tau}}\norm{U-U_{\tau}}_{U_{\tau}^{-1}} + \frac{1}{2\eta} (\overrightarrow{U}-\overrightarrow{U}_{\tau})^* \nabla^2 R(\Xi)(\overrightarrow{U}-\overrightarrow{U}_{\tau})\\
    &= \hat F_{\tau}(U_{\tau}) - \norm{\nabla \hat f_{\tau}(U_{\tau})}_{U_{\tau}}\norm{U-U_{\tau}}_{U_{\tau}^{-1}} + \frac{1}{2\eta}\norm{U-U_{\tau}}^2_{\Xi^{-1}}\\
    &\geq \hat F_{\tau}(U_{\tau}) - \norm{\nabla \hat f_{\tau}(U_{\tau})}_{U_{\tau}}\norm{U-U_{\tau}}_{U_{\tau}^{-1}} + \frac{1}{8\eta}\norm{U-U_{\tau}}^2_{U_{\tau}^{-1}}\\
    &\geq \hat F_{\tau}(U_{\tau}) - \norm{\nabla \hat f_{\tau}(U_{\tau})}_{U_{\tau}} + \frac{1}{2\beta}\,.
\end{align*}
The first inequality follows from Cauchy-Schwartz, the second to last inequality follows from Lemma \ref{lem: mat norm to ratio} and the last inequality follows from the definition of $U$ and the constraint on $\eta$ given by \eqref{eq: eta constraints}. 

Finally, note using Lemma \ref{lem: surrogate func grads} that $\nabla \hat f_{\tau}(U_{\tau})= -\left(1+\beta-\beta \frac{\ip{U_{\tau},R_{\tau}}}{\ip{X_{\tau},R_{\tau}}}\right)\frac{R_{\tau}}{\ip{X_{\tau},R_{\tau}}}$.
Since no reset is triggered at time $\tau-1$, we have using Lemma \ref{lem: U bound} that  $0 \leq \frac{\ip{U_{\tau},R_{\tau}}}{\ip{X_{\tau},R_{\tau}}} \leq  \frac{1}{\beta}$. Therefore we have that 
\begin{align*}
    \norm{\nabla \hat f_{\tau}(U_{\tau})}_{U_{\tau}}=  \left(1+\beta-\beta \frac{\ip{U_{\tau},R_{\tau}}}{\ip{X_{\tau},R_{\tau}}}\right)\frac{\sqrt{\Tr(U_{\tau}R_{\tau}U_{\tau}R_{\tau})}}{\ip{X_{\tau},R_{\tau}}}\leq \left(1+\beta-\beta \frac{\ip{U_{\tau},R_{\tau}}}{\ip{X_{\tau},R_{\tau}}}\right)\frac{\ip{U_{\tau},R_{\tau}}}{\ip{X_{\tau},R_{\tau}}}.
\end{align*}
Maximizing the above expression over all choice of $\frac{\ip{U_{\tau},R_{\tau}}}{\ip{X_{\tau},R_{\tau}}} \in [0, 1/\beta]$ we get that 
\begin{align*}
    \norm{\nabla \hat f_{\tau}(U_{\tau})}_{U_{\tau}} \leq \frac{(1+\beta)^2}{4\beta} \leq \frac{1}{2\beta},
\end{align*}
which follows by the constraint on $\beta$ in \eqref{eq: beta constraints}. Using this and plugging it into the bound for $\hat F_{\tau}(U)$ completes the proof.
\end{proof}

Finally we provide some loose upper bounds on the inverses of the iterates. 
\begin{lemma}
\label{lem: x range}
    If $\eta$ and $\beta$ satisfy constraints~\eqref{eq: eta constraints} and \eqref{eq: beta constraints},
    we have that for any $\tau$ such that the reset condition is not triggered upto index $\tau-1$, 
    \begin{align*}
        &U_{\tau}^{-1} \preceq \left(\frac{(1+\beta)^2T}{4\beta}\eta+d)\right)\identity\\
        \mbox{and }&X_{\tau}^{-1}\preceq P_{\tau}^{-1} \preceq \left(\frac{(1+\beta)^2T}{4\beta^2}\eta+\frac{d}{\beta})\right)\identity
    \end{align*}
    If further $T$ satisfies constraint~\eqref{eq: T constraints}, then the last term is upper bounded by $T^2\identity$.
\end{lemma}
\begin{proof}
Lemma \ref{lem: complex minima} shows that for all Hermitian matrices $H$, such that $\Tr(H) = 0$ we have that $\ip{\nabla \hat F_{\tau - 1}(U_{\tau}), H}$. Further by definition $\hat F_{\tau - 1}(U_{\tau})$ is Hermitian. These facts imply that $\nabla \hat F_{\tau - 1}(U_{\tau}) = \gamma \identity$ for some $\gamma \in \mathbb{R}$. Substituting the definition of $\nabla \hat F_{\tau - 1}$ we get that 
\begin{align}
    \gamma \identity &= \sum_{s=1}^{\tau-1}\left( -\frac{(1+\beta)R_s}{\ip{X_s,R_s}}+\beta\frac{\ip{U_\tau,R_s}R_{s}}{\ip{X_s,R_s}^2}\right)-\eta^{-1}U_{\tau}^{-1}. \label{eq: tempcalc}
\end{align}
Using Lemma \ref{lem: U bound} we get that $\ip{U_\tau, R_s}\leq \frac{\ip{X_s,R_s}}{\beta}$ and therefore the above equality implies that 
\[ \eta^{-1}U_{\tau}^{-1} \preceq -\gamma \identity\]
Further using \eqref{eq: tempcalc} we have that 
\begin{align*}
    \gamma = \ip{U_{\tau} , \nabla \hat F_{\tau-1}(U_{\tau})} &= \ip{U_{\tau} , \left( \sum_{s=\tau_{i-1}}^{\tau-1}\left( -\frac{(1+\beta)R_s}{\ip{X_s,R_s}}+\beta\frac{\ip{U_\tau,R_s}R_{s}}{\ip{X_s,R_s}^2}\right)-\eta^{-1}U_{\tau}^{-1} \right)} \\
    &= \sum_{s=\tau_{i-1}}^{\tau-1}\left( -\frac{(1+\beta)\ip{U_{\tau}, R_s}}{\ip{X_s,R_s}}+\beta\frac{\ip{U_\tau,R_s}^2}{\ip{X_s,R_s}^2}\right)- \frac{d}{\eta} \\
    &\geq -\frac{(1+\beta)^2T}{4\beta}-\frac{d}{\eta}
\end{align*}
Combining these leads to
\begin{align*}
    U_{\tau}^{-1}\preceq  (\frac{(1+\beta)^2T}{4\beta}\eta+d )\identity 
\end{align*}
By the reset condition we have $P_{\tau}\preceq \frac{1}{\beta}U_{\tau}^{-1}$, which completes the first part of the lemma.
Finally, since $\eta\leq\beta/4$, $(1+\beta)^2\leq 3$ and $T \geq \max(2d, \beta^{-1})$, we have
\[\frac{(1+\beta)^2T}{4\beta^2}\eta+\frac{d}{\beta}\leq T^2\,.\]
\end{proof}

\subsection{Main proofs}
For the next three lemmas, once again for brevity we drop the epoch superscript. Further we define the inherent dimension of the problem $\tilde d$ as $d$ for the standard optimal portfolio case and $d^2$ for the quantum case. The next lemma bounds the cost of bias in our algorithm. 
\begin{lemma}
\label{lem: rho sum}
Let $\eta,\beta,T$ satisfy constraints~\eqref{eq: eta constraints}-\eqref{eq: T constraints}. Consider any epoch $e$ with the reset points $\mathcal{T}_{e-1} < \mathcal{T}_e\leq T$. Let $L$ represent the length of the epoch, i.e. $L = \mathcal{T}_e - \mathcal{T}_{e-1}$. Then the cost of bias within the epoch is bounded as follows 
\begin{align*}
    \sum_{\tau=1}^{L}\ip{X_{\tau},P_{\tau}-P_{\tau-1}}=\sum_{\tau=1}^{L}\ip{P_{\tau}^{-1},P_{\tau}-P_{\tau-1}} \leq 2d\log(T)\,.
\end{align*}
\end{lemma}

\begin{proof}
By using Lemma~\ref{lem: telescoping sum} and Lemma~\ref{lem: sum of P}, we have
\begin{align*}
    \sum_{\tau=1}^{L}\ip{X_{\tau},P_{\tau}-P_{\tau-1}} = \sum_{\tau=1}^{L}\ip{P_{\tau}^{-1},P_{\tau}-P_{\tau-1}}\leq \log\det(P_{L}/d)\,.
\end{align*}
Finally by Lemma~\ref{lem: x range}, we have $P_{L}\preceq T^2 \identity$, which completes the proof.
\end{proof}
The following lemma bounds the regret with respect to biased surrogate functions $g_{\tau}$ within an epoch. 
\begin{lemma}
\label{lem: g regret}
    Let $\eta,\beta,T$ satisfy constraints~\eqref{eq: eta constraints}-\eqref{eq: T constraints}. Consider any epoch $e$ with the reset points $\mathcal{T}_{e-1} < \mathcal{T}_e\leq T$. Let $L$ represent the length of the epoch, i.e. $L = \mathcal{T}_e - \mathcal{T}_{e-1}$.  The FTRL-regret with respect to any comparator $U$ over the functions $(g_{\tau})_{\tau=1}^{L}$ is bounded by
    \[
    \sum_{\tau=1}^{L} (g_{\tau}(X_{\tau})-g_{\tau}(U)) \leq \frac{11\tilde d}{\beta}\log(T) + \eta^{-1} R(U)\,,
    \]
    where $\tilde d$ is the inherent dimension of the problem: $d^2$ for the PSD case and $d$ for the simplex case.
\end{lemma}
\begin{proof}
It can be verified that the conditions for Lemma~\ref{lem: ftrl regret upper bound} are satisfied with factor $(1+6\eta)^2 \leq \frac{6}{5}$ due to combining Lemma~\ref{lem: x stability} and \ref{lem: mat norm to ratio}. 

Recall that we denote the canonical vectorization of a matrix $X$ as $\overrightarrow{X}$. Further denote $\overrightarrow\nabla$ as the gradient with respect to this vectorization.
In an overload of notation, we define for a vector $\overrightarrow{X}\in\bbR^{d^2}$ and PSD matrix $M \in \bbR^{d^2\times d^2}$, the semi-norm $\norm{\overrightarrow{X}}_{M}=\sqrt{\ip{\overrightarrow{X},M\overrightarrow{X}}}$ (recall for matrices $X,M\in\bbR^{d\times d}$, we defined $\norm{X}_M=\sqrt{\Tr(XMXM)}$) then 
\begin{align*}
     &\sum_{\tau=1}^L (g_{\tau}(X_{\tau})-g_{\tau}(U)) - \eta^{-1} R(U) \leq \frac{3}{5}\sum_{\tau=1}^L\norm{\overrightarrow{\nabla} g_{\tau}(X_{\tau})}^2_{(\nabla^2G_{\tau}(X_{\tau}))^{-1} }\\
     &\leq \sum_{\tau}^{L}\left(\norm{\overrightarrow{\nabla} \hat f_{\tau}(X_{\tau}) }^2_{(\nabla^2G_{\tau}(X_{\tau}))^{-1}}+\frac{3}{2}B^2\norm{\overrightarrow{P}_{\tau}-\overrightarrow{P}_{\tau-1}}^2_{(\nabla^2G_{\tau}(X_{\tau}))^{-1}}\right)\,,
\end{align*}
where we used $(a+b)^2\leq \lambda a^2 +\frac{\lambda}{\lambda-1}b^2$ for any $\lambda > 1$, generalizing it appropriately to vectors. We deal with the above two terms separately. To control the first term we note using Lemma \ref{lem: surrogate func grads} that
 \[\nabla^2G_{\tau}(X_{\tau}) =  \eta^{-1}\nabla^2 R(X_{\tau}) +\sum_{s=1}^{\tau} \frac{\beta}{2} \overrightarrow{\nabla} \hat f_{s}(X_{s})\overrightarrow{\nabla} \hat f_{s}(X_{s})^*.\]
 Using Lemmas \ref{lem: R function grad}, \ref{lem: R Hessian lower bound} we get that for any $\tau$
 \[\norm{\overrightarrow{\nabla} \hat f_{\tau}(X_{\tau}) }^2_{(\nabla^2G_{\tau}(X_{\tau}))^{-1}} \leq \big \langle \overrightarrow{\nabla} \hat f_{\tau}(X_{\tau}) [\overrightarrow{\nabla} \hat f_{\tau}(X_{\tau})]^*, \left( \eta^{-1}\tidentity+\sum_{s=1}^{\tau} \frac{\beta}{2} \overrightarrow{\nabla} \hat f_{s}(X_{s})\overrightarrow{\nabla} \hat f_{s}(X_{s})^*\right)^{-1} \big \rangle. \]
 Using Lemma \ref{lem: telescoping sum}, the following computation follows. 
\begin{align*}
\sum_{\tau=1}^{L}\norm{\nabla \hat f_{\tau}(X_{\tau}) }^2_{(\nabla^2G_{\tau}(X_{t}))^{-1}} &\leq \frac{2}{\beta}\log\det(\tidentity+\sum_{\tau=1}^L \frac{\eta\beta}{2} \overrightarrow{\nabla} \hat f_{\tau}(X_{\tau})\overrightarrow{\nabla} \hat f_{\tau}(X_{\tau})^*)\\
&\leq\frac{2}{\beta} \tilde d \log \left((\tilde d + \frac{\eta\beta}{2} T\max_{t\in[\tau]}\norm{\overrightarrow{\nabla}\hat f_{\tau}(X_{\tau})}^2)/\tilde d\right)\,.
\end{align*}
By Lemma~\ref{lem: x range}, we have $X_{\tau}\succeq T^{-2}\identity$ for all $\tau\in[L]$, hence
\begin{align*}
    \norm{\overrightarrow{\nabla}\hat f_{\tau}(X_{\tau})}^2=\frac{\ip{R_{\tau},R_{\tau}}}{\ip{X_{\tau},R_{\tau}}^2}\leq T^4\,.
\end{align*}
Further, since $T>2d$ and $\eta\leq \frac{1}{2}$, we have
\[
\sum_{\tau=1}^{L}\norm{\overrightarrow{\nabla} \hat f_{\tau}(X_{\tau}) }^2_{(\nabla^2G_{\tau}(X_{t}))^{-1}} \leq \frac{10}{\beta}\tilde d \log(T)\,.
\]
For the second norm, we have
\begin{align*}
    \sum_{t=1}^\tau\norm{\overrightarrow P_{\tau}-\overrightarrow P_{\tau-1}}^2_{(\nabla^2 G_{\tau}(X_{t}))^{-1}}&\leq \eta\sum_{t=1}^\tau \norm{\overrightarrow P_{\tau}-\overrightarrow P_{\tau-1}}^2_{(\nabla^2 R(X_{\tau}))^{-1}}\\
    &=\eta\sum_{t=1}^\tau \norm{P_{\tau}-P_{\tau-1}}^2_{X_{\tau}}\mbox{ (by Lemma~\ref{lem: R function grad})}\\
    &=\eta\sum_{t=1}^\tau\Tr  \left((\identity - \sqrtX_{\tau}P_{\tau-1}\sqrtX_{\tau})_+^2\right) \\
    &\leq
    \eta\sum_{t=1}^\tau\Tr  \left((\identity - \sqrtX_{\tau}P_{\tau-1}\sqrtX_{\tau})_+(\identity - \sqrtX_{\tau}X_{\tau-1}^{-1}\sqrtX_{\tau})_+\right)\\
    &\leq
    6\eta^2\sum_{t=1}^\tau\Tr  \left((\identity - \sqrtX_{\tau}P_{\tau-1}\sqrtX_{\tau})_+\right)\\
    &=
    6\eta^2\sum_{t=1}^\tau\ip{  X_{\tau}, X_{\tau}^{-1/2}(\identity - \sqrtX_{\tau}P_{\tau-1}\sqrtX_{\tau})_+X_{\tau}^{-1/2}}\\
    &= 
    6\eta^2\sum_{t=1}^\tau\ip{X_{\tau},P_{\tau}-P_{\tau-1}}
    \\
    &= 
    6\eta^2\sum_{t=1}^\tau\ip{P_{\tau}^{-1},P_{\tau}-P_{\tau-1}} \leq 12\eta^2 d \log(T)\,,
\end{align*}
where we use $\sqrtX_{\tau}X_{\tau-1}^{-1}\sqrtX_{\tau}\succeq \frac{1}{1+6\eta}\identity$ by Lemma~\ref{lem: X stability} for the third inequality and  Lemma~\ref{lem: rho sum} for the last equality. 
By $\eta\leq \frac{1}{4B}$, $\beta \leq \sqrt{2}-1$ by constraint~\eqref{eq: eta constraints},\eqref{eq: beta constraints}, we have
\[\frac{3}{2}B^2\sum_{t=1}^\tau\norm{\overrightarrow P_{\tau}-\overrightarrow P_{t-1}}^2_{(\overrightarrow\nabla^2 G_{\tau}(X_{t}))^{-1}}\leq \frac{36d}{32}\log(T)\leq \frac{\tilde d}{\beta}\log(T).\]
Combining both bounds completes the proof.
\end{proof}
The following lemma lower bounds the negative regret contribution we get. 
\begin{lemma}
    \label{lem: negative regret}
    If $\eta,\beta$ satisfy constraints~\eqref{eq: eta constraints} and \eqref{eq: beta constraints}, then for any $\tau$, the negative regret is bounded by
    \begin{align*}
       -\ip{U_{\tau+1},P_\tau-P_0}B \leq \mathbb{I}\{\mbox{reset happened at }\tau\}(-\frac{5B}{12\beta}+dB)
    \end{align*}
\end{lemma}
\begin{proof}
    If no reset happened at $\tau$, we have $P_\tau\succeq P_0$ and the term is bounded by $0$.
    Otherwise by Lemma~\ref{lem: X stability} and the reset condition, we have
    \[\ip{U_{\tau+1},P_\tau}\geq \frac{1}{1+6\eta}\ip{U_{\tau+1},P_{\tau+1}}\geq\frac{1}{2(1+6\eta)^2\beta}\,.\] 
    By the constraint ~\eqref{eq: eta constraints}, we have $(1+6\eta)^2\leq \frac{6}{5}$. Using $P_0=d\identity$ completes the proof. 
\end{proof}
\begin{proof}{\bf of Theorem~\ref{thm: single epoch} and Theorem~\ref{thm: single epoch psd}.}
We use $\tilde d$ to denote $d^2$ in the full PSD case and $d$ in the regular portfolio case (i.e. all matrices are diagonal matrices). With
\begin{align*}
    B &= \frac{264}{5}\tilde d\log(T)\\
    \beta &= \frac{11\tilde d}{7Bd}\\
    \eta &= \frac{1}{4B}\,,
\end{align*}
the constraints can be seen to \eqref{eq: eta constraints}-\eqref{eq: T constraints} be satisfied. Consider any epoch $e$ with the reset points $\mathcal{T}_{e-1} < \mathcal{T}_e\leq T$. Let $L$ represent the length of the epoch, i.e. $L = \mathcal{T}_e - \mathcal{T}_{e-1}$. We drop the superscript $e$ below for brevity. 
Then for any comparator $U \succeq T^{-1}\identity$, we have that
\begin{align}
    \Reg_e(U)&=\sum_{t = \mathcal{T}_{e-1}}^{\mathcal{T}_e-1}(f_t(X_t)-f_t(U))
    \leq \sum_{\tau=1}^{L}(\underline{\hat f_{\tau}}(X_{\tau})-\underline{\hat f_{\tau}}( U))\tag{by Lemma~\ref{lem: lower surrogate}}\\
    &\leq \max_{U'\in\actionSet}\sum_{t=1}^{\tau}(\underline{\hat f_{\tau}}(X_{\tau})-\underline{\hat f_{\tau}}(U')) -\eta^{-1}R(U')+\eta^{-1}R(U) \notag\\
    &=\sum_{\tau=1}^{L}(\hat f_{\tau}(X_{\tau})-\hat f_{\tau}( U_{\tau+1})) -\eta^{-1}R(U_{\tau+1}) +\eta^{-1}R(U)\tag{by Lemma~\ref{lem: extended minimum} }\\
    &=\sum_{t=1}^{\tau}(g_{\tau}(X_{\tau})-g_{\tau}(U_{\tau+1}))-\eta^{-1}R(U_{\tau+1})+\sum_{t=1}^{\tau}\ip{X_{\tau}- U_{\tau+1},B(P_{\tau}-P_{t-1})}  +\underbrace{\eta^{-1}R(U)}_{\leq \eta^{-1} d \log(T)}\,,\notag\\
    &\leq \frac{11 }{\beta}\tilde d\log(T) + 2d\log(T) B+\frac{d\log(T)}{\eta} -(\frac{5B}{12\beta}-dB)\mathbb{I}\{\mbox{reset happened at }\tau\}\tag{by Lemma~\ref{lem: rho sum}-\ref{lem: negative regret}}\\
    &\leq \frac{11}{\beta}\tilde d\log(T)+7d\log(T) B-\frac{5 B}{12\beta}\mathbb{I}\{\mbox{reset happened at }\tau\}\notag\\
    &=\frac{3696}{5}d\tilde d\log^2(T)-\frac{3696}{5}d\tilde d\log^2(T)\mathbb{I}\{\mbox{reset happened at }\tau\}\,.\notag
\end{align}
\end{proof}

\begin{proof}{\bf of Corollary~\ref{cor: regret} and Corollary~\ref{cor:qbisons-regret}.}
We use $\tilde d$ to denote $d^2$ in the full PSD case and $d$ in the regular portfolio case (i.e. all matrices are diagonal matrices). Define $U^\circ = \argmin_{U\in\actionSet}\sum_{t=1}^{T}f_t(X)$ and $U=(1 - \frac{d}{T})U^\circ +\frac{d}{T}(\frac{1}{d}\identity)$. By construction $U \succeq T^{-1}\identity$ is satisfied. As denoted earlier $\mathcal{T}_1,\dots,\mathcal{T}_E$ are the reset points of Algorithm~\ref{alg: bisons (psd)} over the game with $T_0$ steps, and $\mathcal{T}_0 = 1$ and $\mathcal{T}_{E+1}=T+1$ by convention. We now derive the following succession of inequalities
\begin{align*}
    \Reg &\leq \Reg(U) - T\log\left(1-\frac{d}{T}\right) \leq \Reg(U) + \cO(d)\\
    &\leq\sum_{e=0}^{E}\sum_{t\in\mathcal{E}_e}(f_t(X_t)-f_t(U)) + \cO(d) \leq \cO(d\tilde{d}\log^2(T))\,,
\end{align*}
where the first inequality follows via a simple bound on the optimality gap between $U^\circ$ and $U$ and the last step uses the epoch-wise regret bounds established in Theorem \ref{thm: single epoch} and Theorem~\ref{thm: single epoch psd}.
\end{proof}

\begin{proof}{\bf of Lemma~\ref{lem: extended minimum}}
In the proof we omit the superscript $e$ for brevity. Define for any $s$, the set $\mathcal{D}_s:= \{X \in \mathcal{A} | \ip{X, R_s} \leq \beta^{-1} \ip{X_s, R_s}\}$. As we have shown in Section \ref{sec: analysis} we have that for all $s$, $\underline{\hat f_s}|_{\mathcal{D}_s}=\hat f_s|_{\mathcal{D}_s}$ where $l|_{S}$ for a function $l$ and a set $S$ denotes the restriction of the function $l$ on the set $S$.  
The first step is to show the following for any step $\tau$
\begin{align*}
    \underline{D_{\tau}}:=\{U\in\actionSet\,|\,U\preceq \beta^{-1}P_{\tau}^{-1} \}\subset \bigcap_{s=1}^{\tau}\cD_s\,.
\end{align*}
To derive the above, note that due to $U\in\underline{D_{\tau}}$, considering any $s \leq \tau$ and noting that $R_s \succeq 0$, we have
\begin{align*}
    \frac{\ip{U,R_s}}{\ip{X_s,R_s}} &\leq \sup_{R\in \psds}\frac{\ip{U,R}}{\ip{X_s,R}}=\sup_{R'\in \psds}\frac{\ip{U,X_s^{-\frac{1}{2}}R'X_s^{-\frac{1}{2}}}}{\ip{X_s,X_s^{-\frac{1}{2}}R'X_s^{-\frac{1}{2}}}}=\sup_{R'\in \psds}\frac{\ip{X_s^{-\frac{1}{2}}UX_s^{-\frac{1}{2}},R'}}{\Tr(R')}\\
    &=\max_i\ev_i(X_s^{-\frac{1}{2}}UX_s^{-\frac{1}{2}} )  \leq \max_i\ev_i(P_{\tau}^{\frac{1}{2}}UP_{\tau}^{\frac{1}{2}} ) \leq \beta^{-1}\,,
\end{align*}
which concludes that claim.
Next we show that $U_{L+1}\in \operatorname{int}(\underline{D_L})$. Since $L-1$ did not trigger a reset, we know that $U_{L}\prec \frac{1}{2(1+6\eta)\beta}P_L^{-1} $.
By Lemma~\ref{lem: X stability} and \ref{lem: U stability}, we have $U_{L+1} \preceq 2 U_{L}$ and $P_{L}^{-1}\preceq (1+6\eta)P_{L+1}^{-1}$. Hence $U_{\tau+1}\prec  \beta^{-1}P_{\tau+1}^{-1}$.
Finally since $\underline{\hat f_t}|_{\underline{D_t}}=\hat f_t|_{\underline{D_t}}$ and $U_{\tau+1}$ is by definition the minimizer
\begin{align*}
    U_{\tau+1}=\argmin_{X\in\actionSet}\sum_{s=1}^t\hat f_t(X)+\eta^{-1}R(X)\,,
\end{align*}
this implies that $U_{\tau+1}$ is a local minimum and by convexity a global minimum of the LHS in Lemma~\ref{lem: extended minimum}. 
\end{proof}

\section{Solving the \qbisons optimization problem}
\label{sec:solving-opt}

In each iteration of \qbisons (Algorithm~\ref{alg: bisons (psd)}), the main computational effort is in solving the optimization problems
\[X^e_{\tau+1} \leftarrow  \argmin_{X\in\actionSet} G^e_{\tau}(X) \quad \text{ and } U^e_{\tau+1} \leftarrow \argmin_{X\in\actionSet} \hat{F}^e_{\tau}(X).\]
We now show that these can be rewritten as convex minimization problems over a bounded convex subset of $\mathbb{R}^{d^2}$, such that the gradient for the objective can be computed in $\tilde{\cO}(\text{poly}(d))$ time. Also, it suffices to solve these optimization problems to an accuracy of $\frac{1}{\text{poly}(T)}$ with negligible impact on the regret. Hence, the optimization can be done via a method like ellipsoid or Vaidya's algorithm in $\tilde{\cO}(\text{poly}(d))$ time per iteration.

Towards the above goal, we first identify $\Hermitian^d$ with the real space $\mathbb{R}^{d^2}$ simply by enumerating the real and imaginary parts of the $\frac{d^2-d}{2}$ lower triangular entries excluding the diagonal entries, and then the $d$ real diagonal entries. Let $\phi: \Hermitian^d \rightarrow \mathbb{R}^{d^2}$ denote this mapping. It is obvious that $\phi$ is linear and a bijection. Thus, if $f: \Hermitian^d \rightarrow \mathbb{R}$ is a convex function, then $f \circ \phi^{-1}$is also convex. Furthermore, $\phi(\actionSet)$ is a bounded convex set. So it suffices to show that $G^e_{\tau}$ and $\hat{F}^e_{\tau}$ are convex functions on $\actionSet$. We show this for $G^e_{\tau}$, the reasoning for $\hat{F}^e_{\tau}$ is a analogous.
\[G^e_{\tau}(X) = \eta^{-1} R(X) + \sum_{s=1}^\tau \hat f^e_s(X) - \ip{X, P^e_\tau - P^e_0}B.\]
It is well-known that the log-det regularizer $R$ is convex over $\Hermitian^d$ (one way to see that is to use the fact $-\log \det(X) = -\Tr(\log(X))$, and then use the operator concavity of $\log(X)$, which follows from L\"owner's theorem~\citep{Lowner1934}). The last term $- \ip{X, P^e_\tau - P^e_0}B$ is linear and therefore convex, so it remains to show that $\hat f^e_s(X)$ is convex for any $s$. From the definition of $\hat f^e_s$ in \eqref{eq: surrogate loss} we see that we only need to show that $\ip{X, \nabla f_t(X_t)}^2$ is convex. But this follows because $X \mapsto \ip{X, \nabla f_t(X_t)}$ is a linear function of $X$ mapping $X$ to a real number since $\nabla f_t(X_t)$ is Hermitian, and $u \mapsto u^2$ is convex over real numbers.

Finally, turning to gradient computation for $G^e_{\tau}$, note that $\nabla R(X) = -X^{-1}$ which can be computed in $\tilde{\cO}(d^3)$ time. Then, it is easy to see that we can combine all the quadratic surrogate functions $\hat{f}^e_\tau$ into a single quadratic function that we can maintain in $\tilde{\cO}(\text{poly}(d))$ memory over the iterations, and thus we can compute the gradient of $\sum_{s=1}^\tau \hat f^e_s(X)$ in $\tilde{\cO}(\text{poly}(d))$ time as well. The gradient of $- \ip{X, P^e_\tau - P^e_0}B$ is just $B(P^e_\tau - P^e_0)$. Thus, we can compute gradients of $G^e_{\tau}$ in $\tilde{\cO}(\text{poly}(d))$ time.

% \[G^e_{\tau}(X) = \eta^{-1} R(X) + \ip{X, L_\tau} + X_{\mathrm{flat}}^* H_\tau X_{\mathrm{flat}} + C_\tau,\]
% where $C_\tau$ is a constant, $L_\tau \in \Hermitian^d$ and $H_\tau = \frac{\beta}{2}\sum_{s=1}^\tau [\nabla f^e_s(X_s)]_{\mathrm{flat}}[\nabla f^e_s(X_s)]_{\mathrm{flat}}^*$ is a $d^2 \times d^2$ PSD Hermitian matrix.

We note that the running time can be further improved by using Newton's method since the functions $G^e_\tau$ and $\hat{F}^e_\tau$ are actually self-concordant, since all the component functions (log-det, linear, and quadratic) are self-concordant. 
% Note: it turns out that log-det is self-concordant even after reparameterizing $\Hermitian^d$ by $\mathbb{R}^{d^2}$. This follows from the formula here: https://math.stackexchange.com/questions/2559880/determinant-of-hermitian-matrix and the fact that log-det is self concordant for real PSD matrices and self-concordance is preserved under affine composition.

% and moreover these optimization problems are well-conditioned (strongly convex and smooth), and hence can be solved very efficiently via gradient descent. These facts are fairly straightforward for the special case of online portfolios since the log barrier regularization is strongly convex 

\section{FTRL lower bound omitted proofs.}
\label{sec:ftrl-lb}

First, we prove Lemma~\ref{lem: lb-ftrl regret lower bound}. This lemma follows from Lemma~\ref{lem: ftrl regret lower bound}, since $\Pi \actionSet$ has non-zero volume, and the fact that Assumption~\ref{ass:lower boundedness} holds, as shown by the following lemma:
\begin{lemma}
\label{lem: ftrl stability}
    For any $\eta > 0$, \lbftrl 
    % FTRL over the true loss sequences with regularization $R(x)=-\sum_{i=1}^d \log(x_i)$ and any $\eta$ 
    satisfies Assumption~\ref{ass:lower boundedness} with $c_2=\frac{1}{(1+\eta)^2}$.
\end{lemma}
\begin{proof}
First note that for any $x,y\in\operatorname{int}(\Delta([d])$, we have
\begin{align*}
    \nabla^2_\Pi F_t(x) = \sum_{s=1}^t\frac{(\Pi r_s)(\Pi r_s)^\top}{\ip{x,r_s}^2}+\sum_{i=1}^d\frac{(\Pi \mathbf{e}_i)(\Pi \mathbf{e}_i)^\top}{\ip{x,\mathbf{e}_i}^2} \succeq \min_{i\in[d]}\frac{y_i^2}{x_i^2}\nabla^2_\Pi F_t(y)\,. 
\end{align*}
Hence we need to prove that for $c_2=\frac{1}{(1+\eta)^2}$, we have $\min_{i\in[d]}x_{t,i}/x^\lambda_{t,i} \geq \sqrt{c_2}$.

We have
\begin{align*}
    D_{G^\lambda_t}(x^\lambda_t,x_t)+D_{G^\lambda_t}(x_t,x^\lambda_t)&=\ip{x^\lambda_t-x_t, \nabla G^\lambda_t(x^\lambda_t)-\nabla G^\lambda_t(x_t)}\\
    &= \ip{x^\lambda_t-x_t, -\lambda \nabla f_t(x_t)} \\
    &=\lambda \left(\frac{\ip{x_t^\lambda,r_t}}{\ip{x_t,r_t}}-1\right) \,.
\end{align*}
Let $H^\lambda_t(x) = G^\lambda_t - \sum_{s=1}^{t-1}f_s(x)$, then 
\begin{align*}
    D_{H^\lambda_t}(x^\lambda_t,x_t)+D_{H^\lambda_t}(x_t,x^\lambda_t)&=\ip{x^\lambda_t-x_t, \nabla H^\lambda_t(x^\lambda_t)-\nabla H^\lambda_t(x_t)}\\
    &= \eta^{-1}\sum_{i=1}^d\left(\frac{x_{t,i}^\lambda}{x_{t,i}}+\frac{x_{t,i}}{x_{t,i}^\lambda}-2\right) +\lambda \left(\frac{\ip{x_t^\lambda,r_t}}{\ip{x_t,r_t}}+\frac{\ip{x_t,r_t}}{\ip{x_t^\lambda,r_t}}-2\right)
\end{align*}
Since by construction $\nabla^2G^\lambda_t\succeq \nabla^2H^\lambda_t$, we have
\begin{align*}
    D_{G^\lambda_t}(x^\lambda_t,x_t)+D_{G^\lambda_t}(x_t,x^\lambda_t) \geq D_{H^\lambda_t}(x^\lambda_t,x_t)+D_{H^\lambda_t}(x_t,x^\lambda_t)\,,
\end{align*}
which implies
\begin{align*}
    \lambda\left(1-\frac{\ip{x_t,r_t}}{\ip{x_t^\lambda,r_t}}\right) \geq \eta^{-1}\sum_{i=1}^d\left(\frac{x_{t,i}^\lambda}{x_{t,i}}+\frac{x_{t,i}}{x_{t,i}^\lambda}-2\right)\,.
\end{align*}
Let $z=\argmin_{i\in[d]}\frac{x_{t,i}}{x_{t,i}^\lambda}$, then this results in
\begin{align*}
    (1-z) &\geq \eta^{-1} (z^{-1}-z-2)\\
    \Leftrightarrow\qquad z&\geq \frac{1}{1+\eta}\,, 
\end{align*}
as required.
\end{proof}
In the remainder of this section, we use $f(d,T)=\cO(g(d,T))$, $f(d,T)=\Omega(g(d,T))$ to mean that there exists universal constants $C>c>0$ and $T_0 = \poly(d)$ such that for all $T>T_0$, it holds $f(d,T) \leq C g(d,T)$ and $f(d,T) \geq c g(d,T)$ respectively.
$\poly(x)$ hereby means that there exists some fixed exponent $a\in[0,\infty)$ such that the statement holds for $x^a$.
Finally $f(d,T)=\Theta(g(d,T))$ means $f(d,T)=\cO(g(d,T))$ and $f(d,T)=\Omega(g(d,T))$ hold simultaneously.
Also recall that we assume $T>T_0=\poly(\cT,d)$, specifically we will use $\cT\leq T^\alpha$ throughout this section.

Define the scaling factors
$(c_i)_{i=0}^I=1-2^iT^{-\alpha}$, where $I = \lfloor\frac{1}{3}\log_2(T^\alpha)\rfloor$. 
For $\bm{x}\in\Delta([d])$, we define the ``pulling to the center'' operator ${}^{(s)}$, by $\bm{x}^{(s)}= \Pi^{-1}(c_s\Pi\bm{x})= c_s\bm{x}+(1-c_s)\bm{c}$.

\begin{algorithm2e}
\caption{Sequence for large regret.}
\label{alg: general bad sequence}
\KwIn{$(\bm{t}_i,\bm{o}_i)_{i=1}^\cT,\alpha=\frac{1}{8},T$}
\SetKwFunction{FMain}{move-to-x}
    \SetKwProg{Fn}{Function}{:}{}
\For{$i = 1,2,\dots, \cT $}{
\For{$k=1,\dots,T^{\alpha}$}{
\For{$s=1,\dots,\lfloor\frac{1}{3}\alpha\log_2(T)\rfloor$}{
    \While{$x_t\neq \bm{t}^{(s)}_{i}$}{
    $r_t \leftarrow$  move-to-x($\bm{t}^{(s)}_i$;$F_{t-1}$)\\
            $t\leftarrow t+1$
    }
    $r_t \leftarrow \bm{o}^{(s)}(\bm{t}_i)$\\
    $t\leftarrow t+1$
}
}
}
\Fn{\FMain{$x$ ; $F$}}{
$g\leftarrow \Pi\nabla F(x)$\\
$g\leftarrow \min\{T^{-\frac{1}{2}}/\norm{g}_2,\frac{1}{d\max\{1-\ip{g,\Pi x},0\}} \}g $\\
\textbf{return:}$\Pi^{-1}g$
}
\end{algorithm2e}

\paragraph{Basic calculations:}
By definition $c_i=\Theta(1)$ for all $i\in[I]\cup\{0\}$. Further we have for any $s,s'\in[I]\cup\{0\}$: \[1-c_sc_{s'}=(2^s+2^{s'})T^{-\alpha}-2^{s+s'}t^{-2\alpha}=\Theta(2^{\max\{s,s'\}}T^{-\alpha})\,.\]
For any $\bm{x},\bm{y}\in\Delta([d])$, we have
\[\ip{\bm{x}^{(s)},\bm{y}^{(s')}}=\frac{1}{d}+\ip{\Pi\bm{x}^{(s)},\Pi\bm{y}^{(s')}}=\frac{1}{d}+c_sc_{s'}\ip{\Pi\bm{x},\Pi\bm{y}}=\frac{1-c_sc_{s'}}{d}+c_sc_{s'}\ip{\bm{x},\bm{y}}\,.\]
By the assumption on the sequence, we have for any $j< i$:
\begin{align}
    \ip{\bm{t}^{(s)}_i,\bm{o}^{(s')}_j} =\Omega(\ip{\bm{t}_i,\bm{o}})= \Omega\left(\frac{1}{\poly(d)}\right)\label{eq: wrong target}\\
    \ip{\bm{t}^{(s)}_i,\bm{o}^{(s')}_i} = \frac{1-c_sc_{s'}}{d}= \Theta\left(\frac{2^{\max\{s,s'\}}}{dT^\alpha} \right)\label{eq: on-target}
    %\\
    %\ip{\bm{t}^{(s)}_k,\bm{e}_j} = \Omega(\frac{1-c_s}{d}= \Omega\left(\frac{2^s}{dT^\alpha} \right)\,.\label{eq: regularizer}
\end{align}

\paragraph{Bounding the movement steps.}
The main result of this section is the following Lemma.
\begin{lemma}
\label{lem: num of movement}
The number of movement steps up to time $T$ is bounded by $\cO\left(\poly(d)T^{3\alpha+\frac{1}{2}}\log(T)^2\right)$.
\end{lemma}
In order to prove this Lemma, we first require the following.
\begin{lemma}
\label{lem: movement by gradient}
The while routine over move-to-x for a target $\bm{t}$ up from time $t$ requires $\tau  \leq \frac{2T^{1/2}}{d}\norm{\nabla_\Pi F_t(\bm{t})}_2+1$ steps.
\end{lemma}
\begin{proof}
 We have reached the target, if at time $t+\tau$ is holds $\nabla_\Pi F_{t+\tau}(\bm{t}) = 0$. We select the movement returns $r_s$ for $s\in\{t,\dots,t+\tau-1\}$ such that 
 \begin{align*}
     \norm{\nabla_\Pi F_{s+1}(\bm{t})}_2 = \max\{0,\norm{\nabla_\Pi F_{s}(\bm{t})}_2 - \norm{\nabla_\Pi f_{s+1}(\bm{t})}_2\}\,.
 \end{align*}
 When we cannot reach the target in one step, the norm of the gradient is
 \begin{align*}
     \norm{\nabla_\Pi f_{s+1}(\bm{t})}_2 = \frac{\norm{\Pi r_s}}{\frac{1}{d}+\ip{\Pi \bm{t},\Pi r_s}}\geq \frac{T^{-1/2}}{\frac{1}{d}+T^{-1/2}}\geq \frac{d}{2}T^{-1/2}\,. 
 \end{align*}
 Hence the number of steps $\tau$ until the norm is $0$ is bounded by
 \begin{align*}
     \tau \leq \frac{2T^{1/2}}{d}\norm{\nabla_\Pi F_t(\bm{t})}_2+1\,.
 \end{align*}
\end{proof}
\begin{lemma}
\label{lem: gradient diff movement}
For any movement-return $r$ and any $x,y\in\Delta([d])$, it holds
\begin{align*}
    \norm{\nabla_\Pi f(x;r)-\nabla_\Pi f(y;r)} = \cO(d^2T^{-1})\,.
\end{align*}
\end{lemma}
\begin{proof}
\begin{align*}
    \norm{\nabla_\Pi f(x;r)-\nabla_\Pi f(y;r)}&=|\frac{1}{1/d+\ip{\Pi x,\Pi r}}-\frac{1}{1/d+\ip{\Pi y,\Pi r}} |\norm{\Pi r} \\
    &\leq (\frac{1}{1/d-T^{-\frac{1}{2}}}-\frac{1}{1/d+T^{-\frac{1}{2}}})T^{-\frac{1}{2}}=\frac{2T^{-1}}{1/d^2-T^{-1}}= \cO(d^2 T^{-1})\,,
\end{align*}
where we use that movement returns by construction satisfy $\norm{\Pi r}\leq T^{-\frac{1}{2}}$ and $\norm{\Pi x}\leq \norm{x}\leq 1$ for any $x\in\Delta([d])$.
\end{proof}
\begin{lemma}
\label{lem: largest gradient reg}
For any $x \in \Delta([d])$ and $s\in[I]\cup\{0\}$, the largest possible gradient of any regularizer part $r_i(x^{(s)})=f(x^{(s)};\bm{e}_i), i\in[d]$ is
bounded by
\begin{align*}
    \max_{x\in\Delta([d])} \norm{\nabla_\Pi r_i(x^{(s)})} = \cO\left(d\frac{T^\alpha}{2^s}\right)\,. 
\end{align*}
\end{lemma}
\begin{proof}
\begin{align*}
    \norm{\nabla_\Pi r_i(x^{(s)}} = \norm{\frac{\Pi \bm{e}_i}{\ip{x^{(s)},\bm{e}_i}}} \leq \frac{d}{1-c_s}=d\frac{T^{\alpha}}{2^s}\,,
\end{align*}
where we used 
\[\ip{x^{(s)},\bm{e}_i}=\frac{1-c_s}{d}+c_s\ip{x,\bm{e}_i}\geq \frac{1-c_s}{d}\,.\]
\end{proof}
\begin{proof}[Proof of Lemma~\ref{lem: num of movement}]
For the initial move-to-x, we have $F_0(\bm{t}_1^{(0)})=R(\bm{t}_1^{(0)})$, hence by combining Lemma~\ref{lem: largest gradient reg} and \ref{lem: movement by gradient}, we require $\cO(dT^{\alpha+\frac{1}{2}})$ initial steps.
Afterwards, we need to bound the steps between
any two targets $\bm{t}_k^{(s)},\bm{t}_{k'}^{(s')}$, where $k\leq k'$. Assume this switch happens at time $\tau\leq T$ (since the Lemma statement is concerned with movement steps before time $T$), directly after the agent observed a return $\bm{o}^{(s)}_k$ at target $\bm{t}_k^{(s)}$. Hence
\begin{align*}
    \norm{\nabla_\Pi F_\tau(\bm{t}_{k'}^{(s')})}\leq \norm{\nabla_\Pi f(\bm{t}_{k'}^{(s')};\bm{o}^{(s)}_k)}+\norm{\nabla_\Pi F_{\tau-1}(\bm{t}_{k'}^{(s')})-\nabla_\Pi F_{\tau-1}(\bm{t}_{k}^{(s)})}\,,
\end{align*}
where we use that $\norm{\nabla_\Pi F_{\tau-1}(\bm{t}_{k}^{(s)})}=0$ since the agent was in that point when receiving $r_\tau$.
Splitting the time-steps into movement-returns $\cM_\tau:=\{t\in[\tau]\,|\,\norm{\Pi r_t}\leq T^{-1/2}\}$ and regular returns yields
\begin{align*}
    &\norm{\nabla_\Pi f(\bm{t}_{k'}^{(s')};\bm{o}^{(s)}_k)}  +\norm{\nabla_\Pi F_{\tau-1}(\bm{t}_{k'}^{(s')})-\nabla_\Pi F_{\tau-1}(\bm{t}_{k}^{(s)})}\\
    &\leq\norm{\nabla_\Pi R(\bm{t}_{k}^{(s)})-\nabla_\Pi R(\bm{t}_{k'}^{(s')})}+\norm{\sum_{s\in \cM_\tau} \nabla_\Pi f_s(\bm{t}_{k}^{(s)})-\nabla_\Pi f_s(\bm{t}_{k}^{(s)})}\\
    &\qquad+ T^\alpha\Bigg(\sum_{j=1}^{k-1}\sum_{r=0}^I\left(\norm{\nabla_\Pi f(\bm{t}_{k}^{(s)};\bm{o}_j^{(r)})}+\norm{\nabla_\Pi f(\bm{t}_{k'}^{(s')};\bm{o}_j^{(r)})}\right)\\
    &\qquad\qquad+\sum_{r=0}^I\left(\norm{\nabla_\Pi f(\bm{t}_{k}^{(s)};\bm{o}_k^{(r)})}+\norm{\nabla_\Pi f(\bm{t}_{k'}^{(s')};\bm{o}_k^{(r)})}\right)\Bigg)\\
    &\leq  \cO(d^2 T^\alpha) + \cO(d^2)  \tag{Lemma~\ref{lem: largest gradient reg} and \ref{lem: gradient diff movement} }\\
    &\qquad+\cO\left(\max_{j<\ell; r,r'\in[I]\cup\{0\}}\frac{T^\alpha\cT\log(T)}{\ip{\bm{t}_\ell^{(r)},\bm{o}_j^{(r')}}}\right)+\cO\left(\max_{j\leq\ell; r,r'\in[I]\cup\{0\}}\frac{T^\alpha\log(T)}{\ip{\bm{t}_\ell^{(r)},\bm{o}_j^{(r')}}}\right) 
    \\
    &=\cO(\poly(d)T^{2\alpha}\log(T))\tag{Equation~\eqref{eq: wrong target} and \eqref{eq: on-target}}\,.
\end{align*}
The proof is completed by applying Lemma~\ref{lem: movement by gradient}, noting that the number of switches is bounded by $I\cT\leq T^{\alpha}\log(T)$.
\end{proof}

\paragraph{Bounding the Hessian trace.}

We first bound the Hessian trace of movement-steps.
\begin{lemma}
\label{lem: Hessian movement}
    The movement time-steps $\cM_\tau$ for any $\tau\leq T$ and any $\bm{t}\in\Delta([d])$ satisfy
    \begin{align*}
        \sum_{t\in \cM_\tau}\norm{\nabla_\Pi f(\bm{t};r_t)}^2=\cO(d^2)\,.
    \end{align*}
\end{lemma}
\begin{proof}
By construction $\norm{\Pi r_t}\leq T^{-\frac{1}{2}}$, so 
\begin{align*}
    \norm{\nabla_\Pi f(\bm{t};r_t)}^2 = \frac{\norm{\Pi r_t}^2}{(\frac{1}{d}+\ip{\Pi\bm{t},\Pi r_t})^2}\leq\frac{T^{-1}}{(\frac{1}{d}-T^{-\frac{1}{2}})^2}=\cO(d^2T^{-1})\,.
\end{align*}
Summing over less than $T$ time-steps completes the proof.
\end{proof}

We are ready to bound the total Hessian.
\begin{lemma}
\label{lem: Hessian bound}
    Assume $\tau\leq T$ is  the time-step where the $m$-th iteration through targets $(\bm{t}_i^{(s)})_{s=0}^I$ is completed, then the trace of the Hessian at any target $\bm{t}_{i}^{(s)}$ is bounded by
    \[\Tr(\nabla^2_\Pi F_\tau(\bm{t}_i^{(s)})) = \cO\left( (\poly(d)+m(s+1)\right)d^2\frac{T^{2\alpha}}{2^{2s}}\norm{\Pi\bm{o}_i}^2\,.\]
\end{lemma}
\begin{proof}
We split the trace into 4 terms below based on various contributions from (a) the regularizer,
% \begin{align*}
%     \Tr(\nabla_\Pi^2 R(x))=\sum_{i=1}^d\norm{\nabla_\Pi f(x;\bm{e}_i)}^2\,,
% \end{align*}
(b) the time steps $\cM_\tau$ where the returns are movement-returns selected by the move-to-x subroutine, (c) the returns $\bm{o}^{(s)}_j$ selected for $j<i$ and (d) the returns selected for targets $\bm{t}_i^{(s)}, s\in[I]\cup\{0\}$.
The first two terms are bounded by Lemma~\ref{lem: largest gradient reg} and \ref{lem: Hessian movement} respectively. 
\begin{align*}
\Tr(\nabla^2_\Pi F_\tau(\bm{t}_i^{(s)}))&=\sum_{i=1}^d\norm{\nabla_\Pi f(\bm{t}_i^{(s)};\bm{e}_i)}^2 +\sum_{s\in\cM_\tau}\norm{\nabla_\Pi f_s(\bm{t}_{i}^{(s)})}^2\\
&\qquad+\sum_{j=1}^{i-1}\sum_{s'=0}^I\norm{\nabla_\Pi f(\bm{t}_i^{(s)};\bm{o}^{(s')}_j)}^2+m\sum_{s'=0}^I\norm{\nabla_\Pi f(\bm{t}_i^{(s)},\bm{o}^{(s')}_i)}^2 \\
    &\leq d^3\frac{T^{2\alpha}}{2^{2s}}+\cO(d^2)+\max_{j<i,s'\in[I]\cup\{0\}}\frac{\cT I}{\ip{\bm{t}_i^{(s)},\bm{o}_j^{(s')}}^2} +m\sum_{s'=0}^I\frac{\norm{\Pi\bm{o}_i}^2}{\ip{\bm{t}_i^{(s)},\bm{o}_i^{(s')}}^2}\\
    &\leq\cO\left(d^3\frac{T^{2\alpha}}{2^{2s}}\right)+\cO(\poly(d)T^\alpha\log(T))\\
    &\qquad+m\left(\sum_{s'=0}^s 2^{-2s}+\sum_{s'=s+1}^I2^{-2s'}\right)d^2T^{2\alpha}\norm{\Pi\bm{o}_i}^2\\
    &=\cO\left(\poly(d)\frac{T^{2\alpha}}{2^{2s}}\right)+\cO\left(m(s+1)d^2\frac{T^{2\alpha}}{2^{2s}}\right)\norm{\Pi\bm{o}_i}^2\\
    &\qquad\\\,,
\end{align*}
where we use equations~\eqref{eq: wrong target} and \eqref{eq: on-target} and the fact that $\cT \leq T^\alpha$. $\cO(T^\alpha\log(T)=\cO(\frac{T^{2\alpha}}{2^{2s}})$follows from 
\[2^{2s}T^{-\alpha}\log(T)\leq T^{-\alpha/3}\log(T)=\cO(1)\,.\] 
Finally, observe
\begin{align*}
    0 = \ip{\bm{t}_i,\bm{o}_i}=\frac{1}{d}+\ip{\Pi \bm{t}_i,\Pi \bm{o}_i}\geq \frac{1}{d}-\norm{\Pi \bm{o}_i}\,.
\end{align*}
Hence 
\[
\cO(\poly(d)\frac{T^{2\alpha}}{2^{2s}}) = \cO(\poly(d)\frac{T^{2\alpha}}{2^{2s}})\norm{\Pi\bm{o}_i}^2\,,
\]
which concludes the proof.
\end{proof}

\subsection{Main lower bound proof}
\begin{proof}[Proof of Theorem~\ref{thm: general lower bound}]
By Lemma~\ref{lem: num of movement}, there are $\cO(\poly(d)T^{3\alpha+\frac{1}{2}}\log^2(T))$ movement-returns before time $T$ and the algorithm walks through $\cO(\cT  I)=\cO(T^\alpha\log(T))$ regular returns, hence for $\alpha = \frac{1}{8}$, $\cO(\poly(d)T^{7/8}\log^3(T)) = \cO(T^{15/16})$ and there exists a sufficiently large $T_0=\poly(d,\cT )$, such that the algorithm finishes before time $T$.

Next we bound the stability term.
We have
\begin{align*}
\norm{\nabla_\Pi f_t(x_t)}^2_{(\nabla_\Pi^2 F_t(x_t))^{-1}} \geq \frac{\norm{\nabla_\Pi f_t(x_t)}^2}{\Tr(\nabla_\Pi^2 F_t(x_t))}\,.
\end{align*}
For the $m$-th time of visiting $\bm{t}_i^{(s)}$, the denominator is by Lemma~\ref{lem: Hessian bound} bounded by $\cO((\poly(d)+m(s+1)))d^2\frac{T^{2\alpha}}{2^{2s}}\norm{\Pi\bm{o}_i}^2$.
For $m\geq T^{\alpha/2}$, the trace bound simplifies to $\cO(m(s+1)d^22^{-2s}T^{2\alpha})\norm{\Pi\bm{o}_i}^2$, since we assume $T^{\alpha/2}=\Omega(\poly(d))$.
The nominator is 
\begin{align*}
\norm{\nabla_\Pi f(\bm{t}_i^{(s)};\bm{o}^{(s)}_i)}^2= \Theta(d^22^{-2s}T^{2\alpha}\norm{\Pi\bm{o}_i}^2)\,.
\end{align*}
For the total stability, we have
\begin{align*}
    (stab) &\geq \sum_{i=1}^{\cT }\sum_{m=T^{\alpha/2}}^{T^\alpha}\sum_{s=0}^I\frac{1}{m(s+1)}\\
    &=\Omega\left(\cT \log(T)\log(I)\right)\,.
\end{align*}
Finally $\log(I)=\Theta(\log\log(T))$ completes the proof.
\end{proof}

\section{Follow-The-Regularized-Leader analysis}
Both our main results rely on the standard analysis for FTRL, which we revisit in this section.
Vanilla FTRL is used for online learning over a convex action set $\cX$, where the environment picks a sequence convex loss functions $(g_t)_{t=1}^T$ from some function space $\cG$. The input to FTRL is a regularizer $R:\mathcal{X}\rightarrow\bbR$ and the algorithm plays
\begin{align*}
    x_t = \argmin_{x\in\cX} G_{t-1}(x) := \argmin_{x\in\cX}\sum_{s=1}^{t-1}g_s(x)+\eta^{-1}R(x)\,.
\end{align*}
We consider in this paper special cases of FTRL that allow for a simple regret analysis.
\begin{assumption}
\label{ass: optimal ftrl point}
The action set $\cX\subset\bbR^{\tilde d}$ is compact and the regularizer  $\nabla R(x)$ is strictly convex, twice continuously differentiable and goes to infinity on the boundary of $\cX$.
\end{assumption}
This assumption is directly satisfies by the simplex $\actionSet=\Delta([d])$ and the log-barrier regularizer.
Furthermore the log loss and log-barrier regularization ensure the following.
\begin{assumption}
There exists a universal constant $c_1$, such that for any sequence of functions $g_1,\dots g_T$, any point $\bar x_t$ on the line between $x_t$ and $x_{t+1}$, satisfies
\begin{align*}
     \nabla^2 G_t(\bar x_t)\preceq c_1\nabla^2 G_t(x_t)\,.
\end{align*}
\label{ass:upper boundedness}
\end{assumption}
\begin{assumption}
There exists a universal constant $c_2$, such that for any sequence of functions $g_1,\dots g_T$, the interpolation between $x_t$ and $x_{t+1}$ defined by
\begin{align*}
x_t^\lambda &:= \argmin_{x\in\cX}G_{t-1}(x)+g_t(x)-(1-\lambda)\ip{x,\nabla g_t(x_t)},
\end{align*}
satisfies for any $\lambda\in[0,1]$
\begin{align*}
    \nabla^2 G_t(x_t^\lambda) \succeq c_2\nabla^2 G_t(x_t)\,.
\end{align*}
\label{ass:lower boundedness}
\end{assumption}
For any FTRL algorithm satisfying the assumptions above, the regret is tightly lower and upper bounded as shown in the following lemmas. 

The following lemma gives an upper bound on the regret. We will prove this lemma even for the quantum case. We refer the reader to Section \ref{sec:app complex definitions} for relevant definitions of gradient, Hessian and Bregman divergences in that setting.  
\begin{lemma}
\label{lem: ftrl regret upper bound}
Under Assumptions~\ref{ass: optimal ftrl point} and \ref{ass:upper boundedness}, the regret of FTRL is upper bounded for any comparator $u$ by
\begin{align*}
    &\sum_{t=1}^T(g_t(x_t)-g_t(u))\leq\frac{c_1}{2}\sum_{t=1}^T\norm{\nabla g_t(x_t)}^2_{(\nabla^2 G_t(x_t))^{-1}}+\frac{R(u)-R(x_1)}{\eta}\,.
\end{align*}
\end{lemma}
\begin{proof} %[of Lemma \ref{lem: ftrl regret upper bound}]
We have
\begin{align*}
    \sum_{t=1}^T (g_t(x_t)-g_t(u)) = \sum_{t=1}^T( G_t(x_t)-G_t(x_{t+1}))+ \underbrace{G_T(x_{T+1})-G_T(u))}_{\leq 0}+\eta^{-1}(R(u)-R(x_1))\,.
\end{align*}

For the upper bound, since $x_{t+1}$ minimizes $G_t$ we have that $\forall x\in\actionSet:\,\ip{x-x_{t+1},\nabla G_t(x_{t+1}}= 0$ (For the quantum learning case this is explicitly derived in Lemma \ref{lem: complex minima}). By Taylor's theorem, there exists $\lambda\in [0,1]$ such that $D_{G_t}(x_{t+1},x_t)=\frac{1}{2}\norm{x_{t+1}-x_t}^2_{\nabla^2G_t(\bar x_t^\lambda)}$ (For the quantum case this statement is explicitly proven in Lemma \ref{lemma: complex ivt}).  Therefore we have that

\begin{align*}
    G_t(x_t)-G_t(x_{t+1})&= D_{G_t}(x_t,x_{t+1}) \\
    &= \ip{x_t-x_{t+1},\nabla G_t(x_{t})-\nabla G_t(x_{t+1})} - D_{G_t}(x_{t+1},x_{t})\\
    &= \ip{x_t-x_{t+1},\nabla g_t(x_{t})} - \frac{1}{2}\norm{x_{t+1}-x_{t}}^2_{\nabla^2G_t(\bar x_t^\lambda)}\\
    &\leq \norm{x_t-x_{t+1}}_{\nabla^2G_t(\bar x_t^\lambda)}\norm{\nabla g_t(x_{t})}_{\nabla^2G_t(\bar x_t^\lambda)^{-1}}- \frac{1}{2}\norm{x_{t+1}-x_{t}}^2_{\nabla^2G_t(\bar x_t^\lambda)}\\
    &\leq \frac{1}{2}\norm{\nabla g_t(x_{t})}_{\nabla^2G_t(\bar x_t^\lambda)^{-1}}^2 \leq \frac{c_1}{2}\norm{\nabla g_t(x_{t})}_{(\nabla^2 G_t(x_t))^{-1}}^2\,.
\end{align*}
The above statement combined with the decomposition above implies the statement of the lemma. 
\end{proof}

\begin{lemma}
\label{lem: ftrl regret lower bound}
If $\actionSet$ has non-zero volume in its embedded space and Assumptions~\ref{ass: optimal ftrl point} and \ref{ass:lower boundedness} are satisfied, then the
 regret is lower bounded by
\begin{align*}
    &\frac{c_2}{2}\sum_{t=1}^T\norm{\nabla g_t(x_t)}^2_{(\nabla^2 G_t(x_t))^{-1
}}\leq \max_{u'\in\cX} \sum_{t=1}^T(g_t(x_t)-g_t(u'))\,.\\
\end{align*}
\end{lemma}
\begin{proof} %[of Lemma~\ref{lem: ftrl regret lower bound}]
We have
\begin{align*}
    \sum_{t=1}^T (g_t(x_t)-g_t(u)) = \sum_{t=1}^T( G_t(x_t)-G_t(x_{t+1}))+ \underbrace{G_T(x_{T+1})-G_T(u))}_{\leq 0}+\eta^{-1}\underbrace{(R(u)-R(x_1))}_{\geq 0}\,.
\end{align*}
For the lower bound, we can simply lower bound $\max_{u'}$ by picking $u'=x_{T+1}$ and omit the last two terms.
It remains to analyse the first term.
Given that $R(x)\rightarrow \infty$ on the boundary of $\cX$, the points $x_{t}$ are all strictly in the interior of $\actionSet$.

\begin{align*}
G_t(x_t)-G_t(x_{t+1})&= D_{G_t}(x_t,x_{t+1})\\
    &=D_{G_t^*}(\nabla G_t(x_{t+1}),\nabla G_t(x_t))\\
    &=\frac{1}{2}\norm{\nabla G_t(x_{t+1})-\nabla G_t(x_t)}^2_{\nabla^2G_t^*((1-\lambda)\nabla G_t(x_t)+\lambda \nabla G_t(x_{t+1}) )}\\
    &= \frac{1}{2}\norm{g_t(x_t)}^2_{(\nabla^2 G_t (x_t^\lambda))^{-1}}\geq \frac{c_2}{2}\norm{g_t(x_t)}^2_{(\nabla^2 G_t (x_t))^{-1}}\,.
\end{align*}
The above statement using the decomposition implies the lemma.  
\end{proof}
\end{document}